%% file: main.tex
\documentclass[lettersize,journal]{IEEEtran}
\usepackage{amsmath,amsfonts}
\usepackage{array}
\usepackage[caption=false,font=normalsize,labelfont=sf,textfont=sf]{subfig}
\usepackage{textcomp}
\usepackage{stfloats}
\usepackage[hyphens]{url}
\usepackage[colorlinks=true,breaklinks=true,linkcolor=black,
            urlcolor=black,
            citecolor=black]{hyperref}
\usepackage{verbatim}
\usepackage{cite}
\usepackage{amsmath,amssymb,amsfonts}
\usepackage{algorithmic}
\usepackage{graphicx}
\usepackage{textcomp}
\usepackage{xcolor}
\def\BibTeX{{\rm B\kern-.05em{\sc i\kern-.025em b}\kern-.08em
    T\kern-.1667em\lower.7ex\hbox{E}\kern-.125emX}}

\usepackage[ruled]{algorithm2e}
\usepackage{url}
\usepackage{booktabs}
\usepackage{graphicx} 

\usepackage{diagbox} 

\usepackage{footnote} 

\usepackage{enumitem} 
\usepackage{setspace} 
\usepackage{balance}  

\usepackage{multirow} 
\usepackage{makecell} 

\usepackage{colortbl}
\usepackage{tikz}
\usepackage[ruled]{algorithm2e}
\usepackage{soul}
\usepackage{amsthm}
\newtheorem{definition}{Definition} 

\usepackage{amsthm}           
\newtheorem{lemma}{Lemma}     
\newtheorem{myEx}{Example}
\begin{document}

\title{\textsc{Moon}: A Modality Conversion-based Efficient Multivariate Time Series Anomaly Detection}

\author{Yuanyuan Yao,
        Yuhan Shi,
        Lu Chen,
        Ziquan Fang, \\
        Yunjun Gao,~\IEEEmembership{Senior Member,~IEEE}, 
        Leong Hou U,
        Yushuai Li,~\IEEEmembership{Senior Member,~IEEE},
        Tianyi Li
    \IEEEcompsocitemizethanks{
        \IEEEcompsocthanksitem{This paper was produced by the IEEE Publication Technology Group. They are in Piscataway, NJ.}
        \IEEEcompsocthanksitem Y. Yao, Y. Shi, L. Chen (Corresponding Author), Z. Fang and Y. Gao are with the College of Computer Science, Zhejiang University, Hangzhou 310027, China, E-mail:\{yoyoyao, shiyuhan, luchen, zqfang, gaoyj\}@zju.edu.cn.
        \IEEEcompsocthanksitem Leong Hou U is with the Department of Computer and Information Science, University of Macau, Macau, E-mail:ryanlhu@um.edu.mo. 
        \IEEEcompsocthanksitem Y. Li and T. Li are with the Department of Computer Science, Aalborg University, Denmark, E-mail:\{yusli, tianyi\}@cs.aau.dk.
    }
}

\markboth{IEEE TRANSACIYONS ON KNOWLEDGE AND DATA ENGINEERING, VOL. XX NO. XX. XXX XXXX}%
{Shell \MakeLowercase{\textit{et al.}}: A Modality Conversion-based Multivariate Time Series Anomaly Detection}


\maketitle

\begin{abstract}
Multivariate time series (MTS) anomaly detection identifies abnormal patterns where each timestamp contains multiple variables.  Existing MTS anomaly detection methods fall into three categories: reconstruction-based, prediction-based, and classifier-based methods. However, these methods face two key challenges: {(1) Unsupervised learning methods, such as reconstruction-based and prediction-based methods, rely on error thresholds, which can lead to inaccuracies; (2) Semi-supervised methods mainly model normal data and often underuse anomaly labels, limiting detection of subtle anomalies;
(3) Supervised learning methods, such as classifier-based approaches, often fail to capture local relationships, incur high computational costs, and are constrained by the scarcity of labeled data.} To address these limitations, we propose \textsc{Moon}, a supervised \underline{m}odality c\underline{o}nversion-based multivariate time series an\underline{o}maly detectio\underline{n} framework. \textsc{Moon} enhances the efficiency and accuracy of anomaly detection while providing detailed anomaly analysis reports. First, \textsc{Moon} introduces a novel multivariate Markov Transition Field (MV-MTF) technique to convert numeric time series data into image representations, capturing relationships across variables and timestamps. Since numeric data retains unique patterns that cannot be fully captured by image conversion alone, \textsc{Moon} employs a Multimodal-CNN to integrate numeric and image data through a feature fusion model with parameter sharing, enhancing training efficiency. Finally, a SHAP-based anomaly explainer identifies key variables contributing to anomalies, improving interpretability. Extensive experiments on six real-world MTS datasets demonstrate that \textsc{Moon} outperforms six state-of-the-art methods by up to 93\% in efficiency, 4\% in accuracy and, 10.8\% in interpretation performance.
\end{abstract}


\input{intro}
\input{prilimary}

\input{overview}

\input{design}
\input{experiment}

\input{related_work}
\input{conclusion}


\bibliographystyle{abbrv}
\bibliography{reference}

\end{document}

%% file: intro.tex
\section{Introduction}
\IEEEPARstart{M}ultivariate time series anomaly detection identifies unusual patterns or behaviors across multiple variables over time. It benefits a wide range of real-life applications, including finance~\cite{finance}, healthcare~\cite{healthcare}, and industrial monitoring~\cite{industry,pinSQL}, where timely detection of anomalies can lead to significant improvements in decision-making and framework reliability~\cite{survey-ts,survey-unsupervised}. In recent years, deep learning techniques have been widely applied to time series anomaly detection and achieved superior performance.  

{In multivariate time series anomaly detection, the scarcity of labeled anomalies and high annotation costs pose major challenges. To address these, unsupervised methods are widely used~\cite{OmniAnomaly, tranAD, USAD, MAD-GAN, CATCH, GDN, LGMAD}, modeling normal behavior and detecting deviations via prediction or reconstruction errors. However, unsupervised methods typically rely on large amounts of continuous and stable normal data as the basis for modeling. In real-world applications, such as in financial or industrial systems, data often fluctuates due to external shocks, making it difficult to obtain stable normal sequences. Moreover, the definition of ``normal" is ambiguous and often requires manual labeling, further limiting the adaptability of these methods. Semi-supervised methods attempt to combine a small amount of anomalous labels with unlabeled data, improving performance while controlling label dependence. However, they still rely heavily on large amounts of stable normal data during the modeling process, thus, their effectiveness is limited under unstable environments. Additionally, the limited number of anomalous samples restricts their ability to discern complex patterns. In contrast, supervised methods, while exhibiting strong discriminative performance when sample labels are available~\cite{surpervised,TSec,svm}, are highly dependent on labels. In scenarios with scarce samples, they are prone to overfitting, leading to insufficient generalization capability.}


The interpretability of anomaly detection methods is also crucial for understanding the root causes of anomalies. {Existing methods~\cite{DAEMON, OmniAnomaly, tranAD, ANDEA, ECGGAN,ad-fine-tune,tensor} often rely on differences between reconstructed and ground-truth data to identify variables causing anomalies.} However, they treat all models as black boxes, ignoring unique structures and offering \underline{\textit{low interpretability}}. Thus, we aim to develop an \underline{\textit{accurate}}, \underline{\textit{efficient}}, and \underline{\textit{interpretable}} anomaly detection framework.

\textbf{Accurate anomaly detection.} {To achieve accurate and robust anomaly detection under limited supervision, it is crucial to effectively leverage scarce anomaly labels, enhancing the model's ability to recognize and model abnormal patterns with minimal annotation cost.} Multimodal data improves performance by combining complementary information from different sources, thereby enhancing the overall result quality. However, collecting multimodal data is often costly and impractical, while numeric time series data is readily available. 

\begin{figure}[t]
      \centering
      \includegraphics[width=0.7\linewidth]{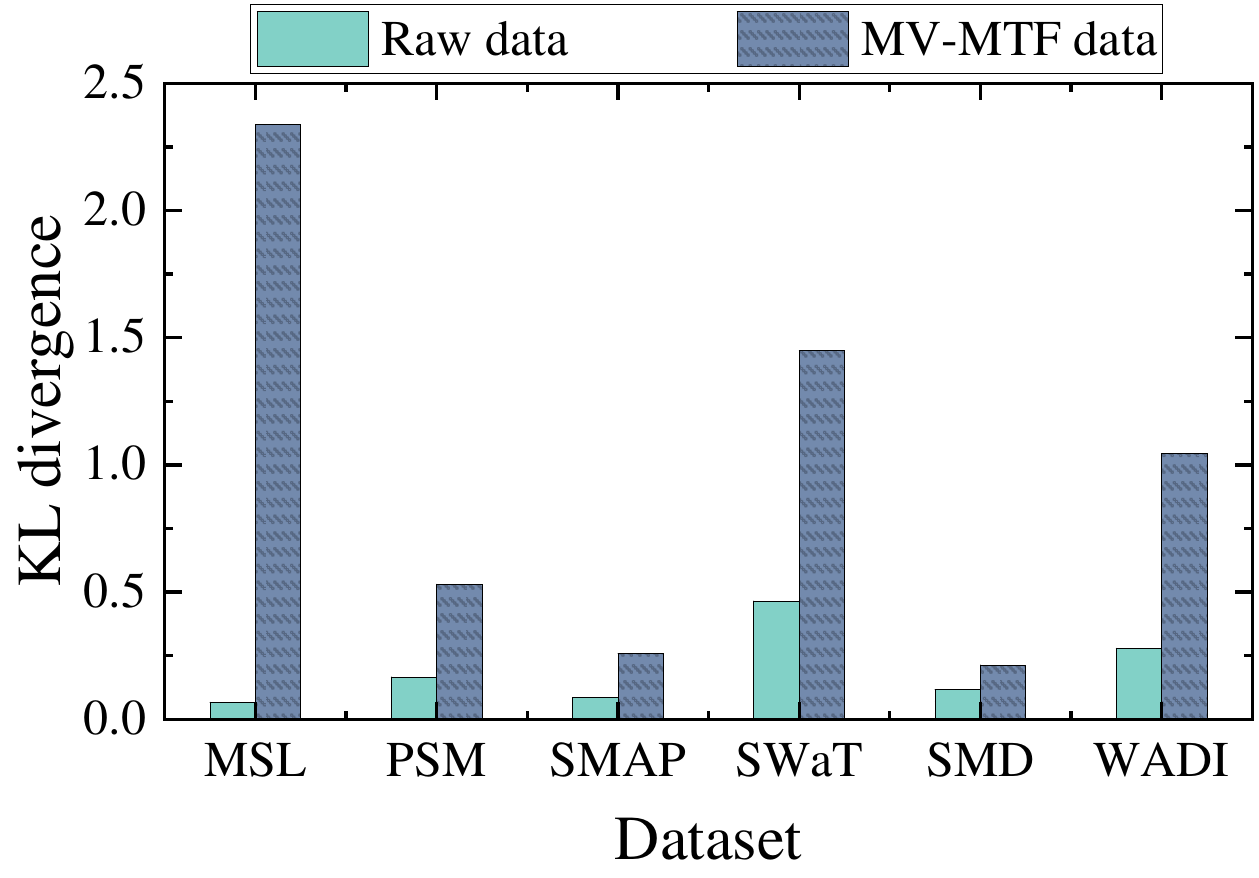}
      \vspace{-4mm}
      \caption{Comparison study on the KL divergence} 
      \label{fig.KL}
      \vspace{-5mm}
\end{figure}

{To address this, we propose Multivariate Markov Transition Field (MV-MTF), which explicitly encodes variable-to-variable and time-to-time transitions in a 2D format to enhance anomaly representation by capturing structural dynamics that are difficult to extract from raw sequences. As shown in Fig.~\ref{fig.KL}, KL divergence analysis shows that MV-MTF enlarges the distribution gap between normal and abnormal data, indicating that its multivariate joint encoding converts cross-variable anomalies into explicit texture patterns, thereby enhancing feature distinguishability. However, MV-MTF may lose fine-grained numerical details that are crucial for identifying point anomalies or small deviations. To compensate for this, we integrate MV-MTF with raw numerical features using a shared-parameter multimodal CNN. This architecture allows the model to jointly learn global structural patterns and local numerical variations~\cite{CNN-CV,CNN-CV-classification}. The multimodal-CNN employs convolutional kernels with varying receptive fields to capture multi-scale features from both the image-like representations and the raw time series. These kernels effectively extract information across variable and timestamp dimensions, enabling robust multimodal feature fusion.} To further enhance the model's ability to process multimodal data, we incorporate a Multimodal Attention mechanism within Multimodal-CNN. This mechanism identifies and focuses on the most relevant features across different modalities, enabling the model to prioritize key information effectively. By integrating this attention mechanism, Multimodal-CNN achieves greater sensitivity and accuracy in detecting anomalous patterns.

\begin{table}[!t]
    \centering \scriptsize
    \setlength{\tabcolsep}{3mm}
    \caption{Modal conversion time (seconds) for MTS} \vspace{-3mm}  
    \begin{tabular}{|c|c|c|c|c|c|}
    \hline
        Datasize   & 4000 & 6000 & 8000 & 10000 \\ \hline
        Extended MTF  & 5428.81 & 12176.4 & 21752.41 & 34862.44 \\ \hline
        MV-MTF & 4.44 & 6.11 & 7.65 & 9.25 \\ \hline
    \end{tabular}
    \vspace{-5mm}
    \label{case_study}
\end{table}

 \textbf{Efficient anomaly detection.}  Efficient modal conversion directly contributes to efficient anomaly detection by reducing the computational overhead associated with processing large datasets. Markov Transition Field (MTF) is an effective modal conversion method that transforms time-series data into image representations for analysis. It captures local temporal information, representing short-term dependencies between consecutive or nearby data points~\cite{MTF15, MTF23}. However, existing MTF methods primarily focus on univariate data~\cite{MTF15, MTF23} and are computationally intensive, as they calculate transitions between every pair of data points. A straightforward extension of MTF to MTS is calculating transitions across different variables and timestamps. As shown in Table~\ref{case_study}, the extended MTF incurs very high processing times, with 10,000 data points requiring nearly 10 hours to process. Moreover, combining data from both modalities may increase the training cost.

To address these limitations, we propose an optimized MV-MTF strategy to improve efficiency of modal conversion, we simplify the time dimension. Since the influence between variables decreases  with temporal distance, we consider only the impact of values from the previous time step on the current time. This reduces computational complexity from \( O(n^2) \) to \( O(n) \), significantly accelerating MV-MTF while preserving essential local information. As shown in Table~\ref{case_study}, with 10,000 data points, MV-MTF processes 10,000 data points in just 9.25 seconds, a 99.97\% reduction compared to Extended MTF. Furthermore, in the multimodal-CNN module, we adopt a parameter-sharing strategy to reduce the overall number of trainable parameters, which not only decreases the memory footprint but also accelerates model convergence. This contributes to a more efficient training process and significantly shortens the overall training time.


\textbf{Interpretable anomaly detection.} {Interpretability is one of the core objectives of our anomaly detection framework. For multivariate time series detected as anomalous, the anomaly typically involves only a subset of variables. Accordingly, solving the anomalies requires not only determining whether an anomaly has occurred but also precisely localizing the affected variables.} Existing threshold-based interpretability methods~\cite{OmniAnomaly, tranAD} face two main limitations in anomaly detection. First, error score threshold-based methods risk misjudgment or omission. Second, they lack intuitive, model-specific explanations, limiting their interpretability.

To address these issues, we propose a high-interpretability method leveraging multimodal information. Gradient Shapley Additive Explanations (SHAP)~\cite{SHAP} values are used to quantify the impact of other variables on a target variable and assess the current variable’s contribution to anomaly occurrence. This enables a multi-dimensional evaluation of anomalies, clarifying how various factors contribute to the detection process.

Next, we apply a weighted ranking method to prioritize variables based on their contributions to anomaly detection, generating a ranked list of the most likely causes. To ensure reliability, we introduce an evaluation module that validates top-ranked variables, reducing false positives and enhancing the robustness of the explanations. Finally, we categorize anomalies into specific types based on expert insights. By clustering the data and constructing classifiers, we identify patterns that refine anomaly categorization. A comprehensive anomaly report is then generated, as exemplified in {Fig.}~\ref{figurereport}.

\begin{figure}[tb]
      \centering
      \includegraphics[width=0.65\linewidth]{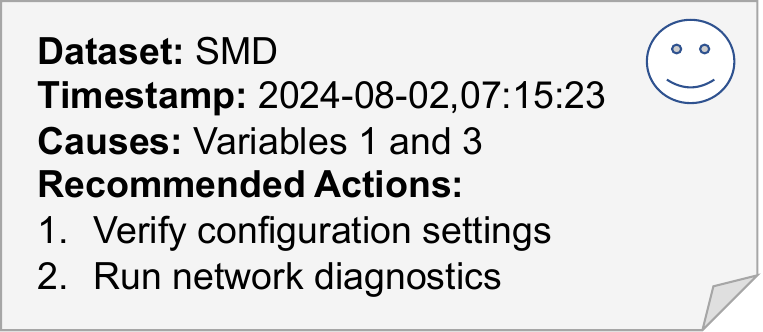}
      \vspace{-3mm}
      \caption{An example of an anomaly report} 
      \label{figurereport}
      \vspace{-3mm}
\end{figure}

The report highlights key information about the anomaly, including (i) the dataset (SMD), (ii) the timestamp (2024-08-02, 07:15:23), (iii) the contributing variables (Variables 1 and 3), and (iv) actionable recommendations (verifying configuration settings and running network diagnostics). The report helps  users quickly understand and address anomalies and enhances framework interpretability and reliability.

We integrate the above novel techniques to propose \textsc{Moon}, a \underline{\textbf{m}}odality c\underline{\textbf{o}}nversion-based an\underline{\textbf{o}}maly detectio\underline{\textbf{n}} framework for MTS that is accurate, efficient, and interpretable. First, \textsc{Moon} employs the Multimodal-CNN to represent and fuse multimodal data while training a classifier for accurate anomaly detection. Next, it efficiently converts numeric MTS data into image data using a new MV-MTF technique. Finally, {\textsc{Moon} integrates an interpretability module to refine anomaly categories and generate detailed anomaly analysis reports.}

To sum up, we have made the following key contributions:

\begin{itemize}
[topsep=0pt,itemsep=0pt,parsep=0pt,partopsep=0pt,leftmargin=*]
\item {We present \textsc{Moon}, a {m}odality c{o}nversion-based MTS an{o}maly detection framework, that efficiently and effectively supports MTS anomaly detection.} 

\item We propose a supervised classification Multimodal-CNN, which integrates numerical MTS data and image representations using convolution kernels with varying receptive fields and a Multimodal Attention mechanism. It multi-scale features and prioritizes key information, enhancing anomaly detection accuracy.

\item We propose a MV-MTF technique designed to simultaneously capture transitions across both time and variables. By leveraging the strong temporal dependencies in time-series data, we further simplify the calculation of transitions between variables, accelerating the modal conversion.

\item We present a SHAP-based anomaly explainer that integrates multimodal data to compute SHAP values and incorporates a reliable evaluation module to enhance accuracy. By categorizing anomalies into distinct types, it significantly improves anomaly detection interpretability.

\item We conduct experiments on six real-world datasets. The results demonstrate that  \textsc{Moon} outperforms six state-of-the-art methods  by up to 93\% in efficiency, 4\% in accuracy and, 10.8\% in interpretation performance.


\end{itemize}

The rest of this paper is organized as follows. We provide the preliminaries in Section 2. Section 3 provides our framework overview and main components of our framework. Section 4 presents the experimental results. We review related work in Section 5, and conclude the paper in Section 6.

%% file: prilimary.tex
\section{Preliminaries}

\subsection{Problem Definition}
\begin{definition}[\textbf{Multivariate time series}] \label{def1}
A multivariate time series $X = (x_1, x_2, \ldots, x_n)$ is a sequence of numerical observations ordered by time, where $x_t$ ($1 \leq t \leq n$) is a $c$-dimensional vector ($c > 1$) representing the observation at timestamp $t$, and $n$ is the length of the time series. The dimension of $X$ is $n \times c$. For the $v^{\textit{th}}$ variable ($1 \leq v \leq c$), $X^v$ represents its univariate time series, and $x^v_t$ denotes its value at timestamp $t$.
\end{definition}

\begin{definition}[\textbf{Anomaly detection}] 
Given a training time series $X$, anomaly detection predicts $Y = \{y_t\}_{t=1}^{\hat{n}}$ for any unseen test time series $\hat{X}$ of length $\hat{n}$ with the same modality as $X$. $y_t \in \{0, 1\}$ indicates whether the data point at timestamp $t$ of $\hat{X}$ is anomalous ($y_t = 1$ denotes anomalous points).
\end{definition}

\begin{definition} [\textbf{Anomaly interpretablity}]
    Given an anomalous time series data point $\hat{x}_t$ at timestamp $t$ with $c$ variables, anomaly interpretability identifies the variables that contributed to classifying $\hat{x}_t$ as anomalous.
\end{definition}

\begin{myEx}
Consider a multivariate time series $X = (x_1, x_2, x_3)$, where each $x_t\,(1 \leq t \leq 3)$ has three variables: temperature, pressure, and humidity. The observations are $x_1 = [30, 101325, 20]$, $x_2 = [32, 101300, 21]$, and $x_3 = [50, 101310, 20]$. Hence, $X^1 = [30, 32, 50]$, $X^2 = [101325, 101300, 101300]$, and $X^3 = [20, 21, 20]$. Anomaly detection identifies $x_3$ as an anomalous data point ($y_3 = 1$). Anomaly interpretability determines that the temperature value $x_3^1 = 50$ is the primary contributor to the anomaly, while pressure $x_3^2 = 101310$ and humidity $x_3^3 = 20$ remain normal.

\end{myEx}

\subsection{Markov Transition Field Technology (MTF)}
MTF techniques~\cite{MTF15, MTF23} are primarily designed for converting univariate time series data to image data.   For a univariate time series $X = (x_1, \ldots, x_n)$ with a single variable ($c = 1$), data points are first discretized by mapping continuous values to discrete bins, reducing computational cost.  Let \( Q \) denote the number of bins used to partition the data range. After discretization, each data point \( x_t \) is assigned to a bin with identifier \( q_i \) (\( i \in [1, Q] \)). A \( Q \times Q \) state transition matrix \( W \) is then constructed, where \( W_{\textit{ij}} \) represents the transition probability from bin \( q_i \) to bin \( q_j \) ($1 \le i, j \le Q$). 


\begin{equation}
    W=\left[\begin{array}{cccc}
w_{11} \mid P_{11} & \cdots & w_{1 Q} \mid P_{1Q} \\
w_{21} \mid P_{21} & \cdots & w_{2 Q} \mid P_{2Q} \\
\vdots &  & \vdots \\
w_{Q 1} \mid P_{Q1} & \cdots & w_{Q Q} \mid P_{QQ}
\end{array}\right],
\end{equation}

\noindent
where \( w_{ij} \) denotes the transition probability in \( P_{ij} \) from \( q_j \) to \( q_i \), with \( P_{ij} = P(x_t \in q_i \mid x_{t-1} \in q_j) \), and \( t \) denotes any given time.

\begin{myEx}\label{ex:univariate}
Given three bins $q_1 = [0, 0.2)$, $q_2 = [0.2, 0.3)$, $q_3 = [0.3, 0.4)$, and a univariate time series $X = (0.1, 0.3,$  $0.05, 0.4, 0.15, 0.2)$, the data points are discretized as follows: $x_1$, $x_3$, and $x_5$ are classified into $q_1$, $x_2$ and $x_4$ into $q_3$, and $x_6$ into $q_2$. The resulting converted time series is $(1, 3, 1, 3, 1, 2)$. For the state transition matrix, $w_{{13}} = 2/3$, where 2 represents the two occurrences of the consecutive pair $(1, 3)$ in the converted series, and 3 is the total occurrences of pairs starting with $1$ ($1 \leq i \leq 3$).
\end{myEx}

Based on the obtained state transition matrix $W$, an \( n \times n \) Markov Transition Field (MTF) matrix \(M \) is constructed to capture the transition probabilities between states (i.e., timestamps), where \( m_{{ij}} \)  represents the transition probability of state \( i \) to the state \( j \), and its value equals to $w_{{ab}}$ (i.e., $x_i \in q_a$, and $x_j \in q_b$). Note that, the obtained MTF matrix $M$ is typically viewed as an image.

\begin{myEx}
Continuing Example~\ref{ex:univariate}, $m_{12} = w_{13}$ due to $x_1 \in q_1$ and $x_2 \in q_3$, $m_{23} = w_{31}$ due to $x_2 \in q_3$ and $x_3 \in q_1$, and $m_{34} = w_{13}$ due to $x_3 \in q_1$ and $x_4 \in q_3$.
\end{myEx}




\subsection{Shapley Additive Explanations (SHAP)}

SHAP~\cite{SHAP} is a unified framework for interpreting the output of machine learning models, grounded in cooperative game theory. The theoretical foundation of SHAP is based on the Shapley value, which provides a fair distribution of payoffs to players depending on their contribution to the total payoff in a cooperative game. In machine learning, SHAP values assign an importance value to each feature, representing its contribution to the model’s output. Given a model $f$ and an instance $x$, the SHAP value $\phi_i(x)$ for feature $i$ is computed as:
\begin{equation}
\phi_i(x) = \sum_{S \subseteq F \setminus \{i\}} \frac{|S|! (|F| - |S| - 1)!}{|F|!} \left[ f_x(S \cup \{i\}) - f_x(S) \right],
\end{equation}
where $F$ is the set of all features, $S$ is a subset of $F$ excluding feature $i$, $|S|$ is the number of features in subset $S$, and $f_x(S)$ is the model output for the subset $S$ with the instance $x$.

The SHAP value computation ensures that the sum of SHAP values for all features equals the difference between the model output and the expected output. If two features contribute equally to all subsets, they receive equal SHAP values. 


%% file: overview.tex
\section{Framework Overview}
\begin{figure*}[ht]
      \centering
      \includegraphics[width=0.93\linewidth]{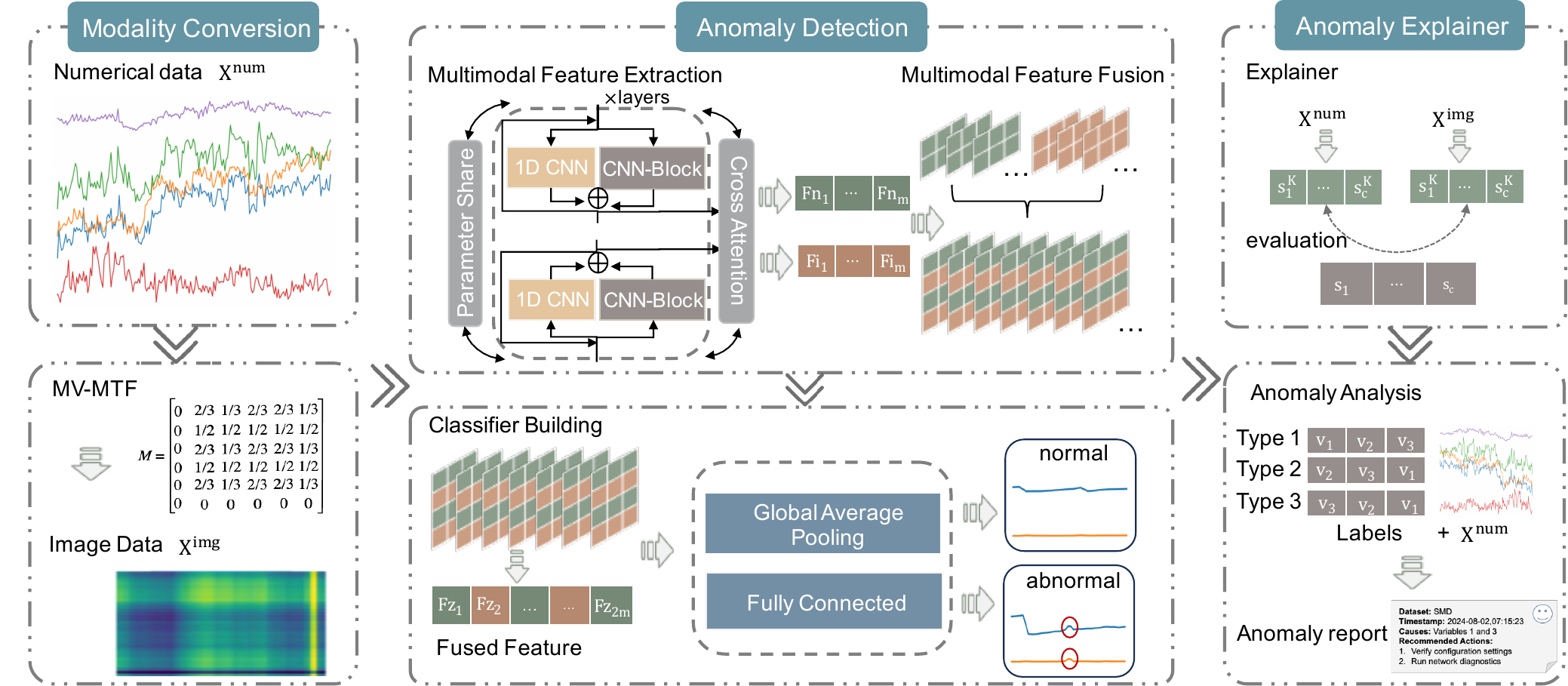}
      \caption{Framework overview} 
      \label{figuresystem}
      \vspace{-4mm}
\end{figure*}
{Fig.}~\ref{figuresystem} illustrates an overview of \textsc{Moon}, comprising modal conversion, anomaly detection, and anomaly explainer.
\begin{itemize}[topsep=1pt,itemsep=0pt,parsep=0pt,partopsep=0pt,leftmargin=*]
\item \textbf{Modal conversion.} To capture the local information between multivariate time series, we use the proposed MV-MTF technique to efficiently calculate the transition probabilities from other variables to a specific variable, thereby converting the multivariate time series into an image.

\item \textbf{Anomaly detection.} To support two different data modalities (i.e., numerical time series data and image data), we propose Multimodal-CNN that utilizes a parameter sharing mechanism to cross-extract features from two modalities, significantly reducing computational complexity and parameter count. In addition, a multi-modal attention mechanism and a separable convolution-based feature fusion technique are introduced to combine the features from both modalities, thereby improving the classifier's performance.
 
\item \textbf{Anomaly explainer.} To provide a user-friendly anomaly explainer, we generate an anomaly detection report for each detected anomly. Specifically, for the anomalous time series data points, we use the kernel explainer and gradient explainer to generate SHAP values based on numerical and image time series data, respectively. Next, we use these weighted SHAP values to identify the variables causing the anomalies. In addition, we train an explainer classifier to refine the anomaly categories.

\end{itemize} 

The online anomaly detection process includes the following three steps. (i) MTS data is converted into image data using  MV-MTF. (ii) Both the original MTS data and the converted image data are simultaneously fed into a Multimodal-CNN to determine if the data is anomalous. (iii) For data identified as anomalous, it is processed by an interpretable classifier to classify the anomaly type and generate a detailed anomaly detection report. 

%% file: design.tex
\section{Framework design}

\subsection{Modality Conversion}
To enable efficient multivariate time series conversion, we introduce MV-MTF technology, which generates an image capturing the time dependency relationships between variables, as shown in the left part of {Fig.}~\ref{figuresystem}.

Let $X^u = \{x^u_1, x^u_2, \ldots, x^u_n\}$ represent the time series for variable $u$ ($u \in [1, c]$), mapped into $Q_u$ distinct bins. Similarly, let $X^{u'} = \{x^{u'}_1, x^{u'}_2, \ldots, x^{u'}_n\}$ represent the time series for another variable $u'$ mapped into $Q_{u'}$ distinct bins. {To determine an appropriate bin count $Q_u$ for each variable $u$, we search over a candidate set of bin counts $Q$. For each $Q_u \in Q$, we apply quantile-based binning to discretize $X^u$ into $Q_u$ bins. The resulting discretized sequence $Q^u$ is then evaluated using entropy:}
\begin{equation}
   {H(Q^u) = -\sum_{i=0}^{|Q_u|-1} p_i \log p_i,\quad p_i = \frac{1}{n} \sum_{t=1}^{n} I(Q_t^u = i)} 
\end{equation}
\noindent
{where $I(\cdot)$ is an indicator function that equals 1 if the $t$-th sample falls into the $i$-th bin, and 0 otherwise. We then select the optimal bin count that maximizes the entropy: $Q_u^* = \arg\max_{Q_u \in Q} H(Q^u)$.}

{With the optimal bin count determined for each variable, we proceed to analyze temporal transitions between discretized states, in contrast to hierarchical binning~\cite{hierarchical-context}, which first partitions the data by layers or groups before discretization.} For two consecutive timestamps, data points $x^{u}_{t-1}$ and $x^{u'}_t$ from variables $u$ and $u'$ are classified into bins $q^{u}_{i}$ ($i \in [1, Q_u]$) and $q^{u'}_{j}$ ($j \in [1, Q_{u'}]$), respectively. A $Q_u \times Q_{u'}$ state transition matrix $W$ is then calculated as follows.

\begin{equation}
   W = \left[ w_{ij} \mid P\left(x_t^{u'} \in q_i^{u'} \mid x_{t-1}^{u} \in q_j^{u}\right) \right],
\end{equation}

\noindent where $1\leq i \leq Q_u$, $1\leq j \leq Q_{u'}$, and $w_{i,j}$ is the transition probability $P( x_t^{u'} \in q^{{u'}}_{j}| x_{t-1}^u \in q^{u}_{i})$ from $q^{u}_{i}$ in variable $X^{u}$ to $q^{{u'}}_{j}$ in variable $X^{v}$.

Based on the obtained state transition matrix $W$, we compute the $n \times n$ Markov Transition Field matrix $M$. For any data point $x^{u}_i$ of variable $u$ at timestamp $i$ and data point $x^{u'}_j$ of variable $u'$ at timestamp $j$ ($x^{u}_i$ is classified into $q^u_a$, while $x^{u'}_j$ is classified to $q^{u'}_b$), the value $m_{ij}$ in $M$ equals to the state transition probability $w_{ab}$ in $W$, which represents the probability of $x^{u'}_t$ at time $t$ belonging to $q^{u'}_{b}$ given that the data in $x_{t-1}^{u}$ at time $t-1$ belonging to $q^{u}_{a}$. This gives us the Markov Transition Field matrix $M_{u,u'}$ for variables $u$ and $u'$.

\begin{myEx}
Given two time series for variables \( u \) and \( u' \): \( {X}^u = (0.1, 0.3, 0.05, 0.4, 0.15, 0.2) \) and \( {X}^{u'} = (0.2, 0.3, 0.45, 0.3, 0.35, 0.4) \), with their respective bins defined as \( Q^u: q^u_1 = [0, 0.2), q^u_2 = [0.2, 0.3), q^u_3 = [0.3, 0.4) \) and \( Q^{u'}: q^{u'}_1 = [0.2, 0.3), q^{u'}_2 = [0.3, 0.4), q^{u'}_3 = [0.4, 0.5) \).

The sequence \( {X}^u \) is classified as follows:
\( x^u_1 \), \( x^u_3 \), and \( x^u_5 \) are classified into \( q^u_1 \);
\( x^u_2 \) and \( x^u_4 \) are classified into \( q^u_3 \);
while \( x^u_6 \) is classified into \( q^u_2 \). For the sequence \( {X}^{u'} \): \( x^{u'}_1 \) is classified into \( q^{u'}_1 \);
\( x^{u'}_2 \), \( x^{u'}_4 \), and \( x^{u'}_5 \) are classified into \( q^{u'}_2 \);
while \( x^{u'}_3 \) and \( x^{u'}_6 \) are classified into \( q^{u'}_3 \). Thus, we obtain the converted time series  \(X_q^u = (1, 3, 1, 3, 1, 2)\) and \(X_q^{u'} =(1, 2, 3, 2, 2, 3)\) respectively. The transition probability \( w_{12} = \frac{2}{3} \), where 2 indicates two occurrences of the pair \((1, 2)\) in consecutive time steps from \( q^u_{t-1} \) to \( q^{u'}_t \), and 3 represents the total occurrences of pairs starting from \( q^u_1 \) (\( 1 \leq i \leq 3 \)). The corresponding matrices \( W \) and \( M \) are shown below.

\begin{gather*}
W = \begin{bmatrix}
0 & 2/3 & 1/3 \\
0 & 0   & 0   \\
0 & 1/2 & 1/2
\end{bmatrix} \\[1em]
M = \begin{bmatrix}
0 & 2/3 & 1/3 & 2/3 & 2/3 & 1/3 \\
0 & 1/2 & 1/2 & 1/2 & 1/2 & 1/2 \\
0 & 2/3 & 1/3 & 2/3 & 2/3 & 1/3 \\
0 & 1/2 & 1/2 & 1/2 & 1/2 & 1/2 \\
0 & 2/3 & 1/3 & 2/3 & 2/3 & 1/3 \\
0 & 0   & 0   & 0   & 0   & 0
\end{bmatrix}
\end{gather*}

\noindent According to \(W\), the probability \( \frac{2}{3} \) from \( q^u_1 \) to \( q^{u'}_2 \) is assigned to \( m_{12} \), since \( x^u_1 \in q^u_1 \) and \( x^{u'}_2 \in q^{u'}_2 \).

\end{myEx}

Note that, we have more than two variables in a multivariate time series data, thus, for a specific variable $u$, the Markov Transition Field matrices $M_{u, k}$ between variable $u$ and all other variables $k$ ($1 \le k \le c$ and $k \neq u$) are computed. These matrices are then averaged to obtain a new Markov Random Field matrix $M_{u, \cdot}$, as shown in Equation~\ref{Equation 6}.
\begin{equation}
    M_{u,\cdot}=\sum_{k=1, k\ne u}^{c} \omega_{u,k} \times M_{u,k},
    \label{Equation 6}
\end{equation}
where $\omega_{u,k}$ is the weight parameter for $M_{u,k}$, denoting the influence degree of $k$ on $u$. The previously computed MTF matrix $M_{u,u}$ for variable $u$ (Equation 2) is weighted and summed with $M_{u,\cdot}$ to obtain the Multivariate Markov Transition Field matrix $M_{u}$ for variable $u$, as shown in Equation~\ref{Equation 7},
\begin{equation}
    M_{u}=\alpha \times M_{u, u} + (1- \alpha)\times M_{u,\cdot},
     \label{Equation 7}
\end{equation}
where $\alpha$ is the weight parameter to control the influence degree of the same variable on the transition probability.

For a multivariate time series data $X$ with $c$ variables, we obtain the MV-MTF $M_X$ for the multivariate time series data $X$ by summing the MTF matrices of all variables, as shown in Equation~\ref{Equation 8}.
\begin{equation}
    M_{X} =\sum_{u=1}^{c} M_{u},
     \label{Equation 8}
\end{equation}
where $M_{u}$ is the MV-MTF matrices for all variables $u$ ($u \in [1, c]$).

 \textbf{Time complexity analysis.} The computational complexity of each MTF matrix computation is $O(n^2)$. In particular, the MTF matrix $M$ requires calculating the values for all pairs in $X^{u}$ and $X^{u'}$, resulting in very high complexity.

However, for time series data, the temporal relevance is crucial, and calculating the relationship between two timestamps time points do not make sense. Therefore, we optimize $M$ as:

\begin{equation}
    M_{i j}=\left\{\begin{array}{ll}
m_{i, j}, & \text { if } j=i+1 \\
0, & \text{ otherwise }
\end{array}\right.,
\label{Equation 9}
\end{equation}

Each element in $M$ represents the transition relationship between two consecutive timestamps. In this way, the complexity of Markov Transition Field matrix computation is reduced from $O(n^2)$ to $O(n)$.


\subsection{Anomaly Detection}

We present Multimodal-CNN with two key components: multimodal feature extraction and multimodal feature fusion.

\begin{figure*}[t]
      \centering
   \includegraphics[width=0.9\linewidth]{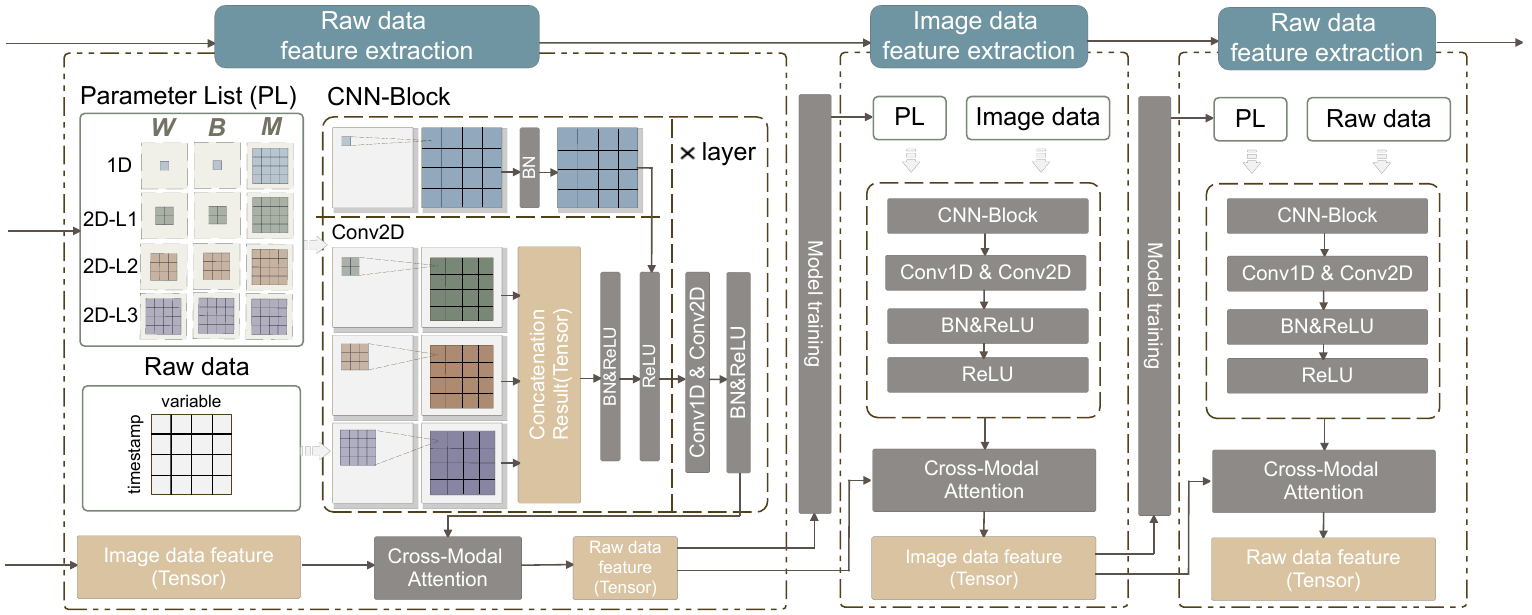}
      \caption{Multimodal feature extraction} 
      \label{figureCNNblock}
      \vspace{-3mm}
\end{figure*}

\subsubsection{Multimodal feature extraction}

This component incorporates a unified framework for extracting features from both numerical and image data, as illustrated in {Fig.}~\ref{figureCNNblock}. It begins by initializing parameters, with a focus on designing the receptive field for the CNN-Block. Using this design, the Multimodal-CNN extracts multi-scale features, capturing critical information through a multimodal attention mechanism. Moreover, parameter sharing is employed across modalities, further enhancing feature extraction performance. The detailed process is described below.

{\underline{\textit{Receptive field design.}} In the multimodal feature extraction module, determining the receptive field size of the 1D-CNN is crucial, as it directly influences the range of features captured during convolution. To handle features at different scales effectively, the receptive field design must be both flexible and efficient. Using convolution kernels of even sizes enables coverage of receptive fields at various scales~\cite{OmiScale-CNN}.

The receptive field (denoted as \(\textit{RF}\)) is calculated as \( \textit{RF} = p(1) + p(2) + p(3) - 2 \). For example, if the kernel size of the first convolutional layer \( p_1 \) is 3 and the second layer \( p_2 \) is 5, the receptive field after a 1D convolution is \(\textit{RF} = 3 + 5 + 1 - 2 = 7 \). Similarly, after a 2D convolution, the receptive field becomes \(\textit{RF} = 3 + 5 + 2 - 2 = 8 \). Extending this concept, ensuring that \( p(1) + p(2) \) covers all even numbers smaller than \( n \) enables the network to achieve receptive fields of all sizes less than \( n \).

To maximize coverage, the network uses sets of prime numbers smaller than \( n \) as convolution kernel sizes for each layer. This  allows processing of receptive fields at different scales, avoiding limitations in local feature extraction and enhancing the network’s ability to capture multi-scale information. In the final layer, two convolution kernels—\( 1 \times 1 \) and \( 2 \times 2 \)—are employed to optimize receptive field coverage. The \( 1 \times 1 \) kernel captures features at the smallest scale, while the \( 2 \times 2 \) kernel captures slightly larger-scale features. This ensures complete coverage of all receptive field sizes, preventing information loss due to incomplete receptive fields and enhancing the network's robustness in feature extraction.



\underline{\textit{CNN-Block and Multimodal attention.} }
Based on the defined parameter list, different convolution kernel sizes are determined for the model. Each convolution layer is responsible for processing different temporal windows, allowing the model to extract information across multiple time scales as shown in Equation~\ref{8}.

\begin{equation}
    \begin{aligned}
        F_{p} &= \texttt{Conv}_{k_{p}\times k_{p}}(X), 
        \\
       F_{\text{concat}} &= \texttt{Concat}(F_{1}, F_{2}, \ldots, F_{p}),
    \end{aligned}
    \label{8}
\end{equation}

\noindent \noindent where \( k_{p} \) is the \( p^{\textit{th}} \) kernel size in each layer, and \( \texttt{Conv}_{k_{p}\times k_{p}}(X) \) denotes the convolution operation applied to the input \( X \) with a kernel size of \( k_{p} \times k_{p} \). This operation extracts local patterns and features within a specific temporal window defined by \( k_{p} \). \texttt{Concat} is used to concatenate the features with different kernel sizes along the channel dimension, forming a consolidated feature representation \( F_{\text{concat}} \).

To preserve information in deeper networks, multimodal feature extraction adopts a residual connection design. Through skip connections, the input data is directly added to the output, ensuring original information is passed to subsequent layers and mitigating the gradient diminishing problem. In each layer, the input passes through both the main network and a $1 \times 1$ convolution layer, which adjusts the channel dimensions to match the extracted features. The residual connection output, $\texttt{Conv}_{1 \times 1}(X)$, is the transformed input, while the combined output, $F_{\text{output}}$, is the sum of the concatenated features $F_{\text{concat}}$ and the residual connection $\texttt{Conv}_{1 \times 1}(X)$, computed as follows.  

\begin{equation}
     F_{\text{output}} = F_{\text{concat}}  \oplus  \texttt{Conv}_{1 \times 1}(X)
     \label{9}
\end{equation}

To enable each input type to focus on relevant features from the other, we calculate cross-type attention weights. For the numerical input, we compute the query $Q_{\text{num}}$, key $K_{\text{num}}$, and value $V_{\text{num}}$ as follows: 
$Q_{\text{num}} = F_{\text{num}} \textbf{W}^{Q}$, $ K_{\text{num}} = F_{\text{num}} \textbf{W}^{K}$, and $V_{\text{num}} = F_{\text{num}} \textbf{W}^{V}$.  
For the image input, we compute the query $Q_{\text{img}}$, key $K_{\text{img}}$, and value $V_{\text{img}}$ as follows:  $ Q_{\text{img}} = F_{\text{img}} \textbf{W}^{Q}$, $\quad K_{\text{img}} = F_{\text{img}} \textbf{W}^{K}$, and $V_{\text{img}} = F_{\text{img}}\textbf{W}^{V}$.  Here, \( \mathbf{W}^{Q} \), \( \mathbf{W}^{K} \), and \( \mathbf{W}^{V} \) are weight matrices learned during training. Next, we compute attention weights to let the numerical data attend to the image data, extracting relevant features:

\begin{equation}
    \begin{aligned}
      \texttt{Attention}_{\text{num} \rightarrow \text{img}} &= \texttt{softmax}\left(\frac{Q_{\text{num}} K_{\text{img}}^\mathrm{T}}{\sqrt{d}}\right), 
       \\
       F_{\text{num} \rightarrow \text{img}} &= \texttt{Attention}_{\text{num} \rightarrow \text{img}} \oplus V_{\text{img}},
    \end{aligned}
    \label{10}
\end{equation}

\noindent where \( Q_{\text{num}} \) is multiplied by \( K_{\text{img}}^\mathrm{T} \) (transpose of \(K_{\text{img}}\)) to compute similarity scores, scaled by \( \sqrt{d} \), where \( d \) is the feature dimension. The result is passed through a \texttt{softmax} function to calculate attention weights \( \texttt{Attention}_{\text{num} \rightarrow \text{img}} \). These weights are then applied to \( V_{\text{img}} \) to produce updated features \( F_{\text{num} \rightarrow \text{img}} \).

Finally, we reverse the process to let the image data attend to the numerical data:

\begin{equation}
    \begin{aligned}
        \texttt{Attention}_{\text{img} \rightarrow \text{num}} &= \texttt{softmax}\left(\frac{Q_{\text{img}} K_{\text{num}}^T}{\sqrt{d}}\right),
        \\
        F_{\text{img} \rightarrow \text{num}} &= \texttt{Attention}_{\text{img} \rightarrow \text{num}} \oplus V_{\text{num}},
    \end{aligned}
    \label{11}
\end{equation}

\noindent where \( Q_{\text{img}} \) is multiplied by \( K_{\text{num}}^\mathrm{T} \) (transpose of \( K_{\text{num}} \)) to compute similarity scores, scaled by \( \sqrt{d} \). The \texttt{softmax} function converts the scores into attention weights, \( \texttt{Attention}_{\text{img} \rightarrow \text{num}} \), which are applied to \( V_{\text{num}} \)  to produce updated features \( F_{\text{img} \rightarrow \text{num}} \). The above process allows each input type to focus on complementary information from the other, enhancing feature representation.

\underline{\textit{Parameter sharing.} } Parameter sharing across different data types reduces the number of parameters and improves training efficiency. The shared parameters include the weights \( \mathbf{W} \), biases \( \mathbf{B} \), and masks \( \mathbf{M} \) for each convolution kernel in the CNN-Block. During each epoch, the model trains sequentially on numerical data and image data, initializing parameters from the previous training iteration. The process is defined as follows:

\vspace{-2mm}
\begin{equation}
\vspace{-1mm}
    \begin{aligned}
     N_b &= f_{b} \left(\mathbf{X}_{n}, (\mathbf{W}_{b}, \mathbf{B}_{b}, \mathbf{M}_{b}) \right) \\
     &\leftarrow f_{b-1}\left(\mathbf{X}_{i}, (\mathbf{W}_{b-1}, \mathbf{B}_{b-1}, \mathbf{M}_{b-1}) \right),
    \\
     I_b &= f_{b}\left(\mathbf{X}_{i}, (\mathbf{W}_{b}, \mathbf{B}_{b}, \mathbf{M}_{b}) \right),
    \end{aligned}
\end{equation}

\noindent where \( \mathbf{X}_n \) and \( \mathbf{X}_i \) are the numerical and image data, respectively; \( b \) denotes the batch index in an epoch; \( \mathbf{W} \), \( \mathbf{B} \), and \( \mathbf{M} \) are the shared weight, bias, and mask parameters of the CNN-Block; \( N_b \) and \( I_b \) are the embeddings generated for the numerical and image data, respectively; and \( f_b \) represents the above multimodal feature extraction process. For each batch \( b \), the numerical data \( \mathbf{X}_n \) is processed first using the parameters \( \mathbf{W}_{b} \), \( \mathbf{B}_{b} \), and \( \mathbf{M}_{b} \) obtained from the image data in the previous batch \( b-1 \). The image data \( \mathbf{X}_i \) is then processed using the parameters updated after training on \( \mathbf{X}_n \) in the current batch. This sequential training  captures correlations between numerical and image data, enhancing the model’s ability to learn effectively from multimodal inputs.

{Parameter sharing not only significantly accelerates gradient convergence but also offers multiple additional benefits. Theoretically, consider gradients from two modalities denoted as \(\mathbf{g}_n\) and \(\mathbf{g}_i\). The gradient update for the shared parameters is a weighted linear combination:}
\begin{equation}
    {\Delta \theta_s = \lambda \mathbf{g}_n + (1-\lambda) \mathbf{g}_i,}
\end{equation}
{where $\lambda \in [0,1]$ is a weighting factor that controls the relative contribution of the two gradients to the final parameter update. The magnitude $\|\Delta \theta_s\|$ denotes the overall step size of the update, determined by the direction and scale of the combined gradients, and can be calculated as: $\|\Delta \theta_s\| = $}
\begin{equation}
   {\sqrt{
\lambda^2 \|\mathbf{g}_n\|^2 + (1-\lambda)^2 \|\mathbf{g}_i\|^2 + 2 \lambda (1-\lambda) \|\mathbf{g}_n\| \|\mathbf{g}_i\| \cos\theta
},} 
\end{equation}
{where $ \cos\theta = \frac{\mathbf{g}_n \cdot \mathbf{g}_i}{\|\mathbf{g}_n\| \|\mathbf{g}_i\|},$ measures the alignment between the two gradient directions. When \(\cos\theta\) approaches to 1 (indicating high directional similarity), the combined gradient magnitude is maximized, resulting in more effective updates and significantly faster convergence.}

{This mechanism enables parameter sharing to promote information fusion between different modalities, allowing the model to learn more generalizable and robust feature representations, thereby improving generalization and noise resilience. Furthermore, the weighted combination of gradients inherently suppresses noise: when one modality’s gradient is noisy, the other can partially offset the disturbance, enhancing training stability and model robustness. Additionally, by reusing the same parameter set, parameter sharing drastically reduces the total number of model parameters, lowering computational and storage costs, while simplifying the model architecture and training process, thus improving overall the training efficiency. }

\subsubsection{Multimodal Feature Fusion} It integrates features from multiple modalities to enhance representation for classification tasks. Multimodal feature fusion includes five key steps: (i) concatenation, (ii) depthwise separable convolution, (iii) layer normalization, and (iv) gated Feedforward Network.

\underline{\textit{Concatenation.}} The features extracted from different modalities are concatenated to form a combined feature vector $F_{\text{concat}} = [F_\text{img}; F_\text{num}]$.  

\underline{\textit{Depthwise separable convolution.}} The concatenated feature vector \( F_{\text{concat}} \) is processed using a depthwise separable convolution. This encompasses two sequential steps: (i) a depthwise convolution extracts spatial or temporal relationships within individual feature channels, and (ii) a point wise convolution integrates information across all channels, enhancing feature interaction and representation. The resulting output \(F_{\text{out}}\) is computed as follows:
\begin{equation}
    F_{\text{out}} = \texttt{PointwiseConv}\left( \texttt{DepthwiseConv}(F_{\text{concat}}) \right) 
\end{equation}

Splitting standard convolution into these two steps significantly reduces computational complexity by minimizing cross-channel operations, while preserving the richness of extracted features through efficient intra-channel pattern extraction and cross-channel integration.

\underline{\textit{Layer normalization.}} Layer normalization stabilizes and accelerates training by normalizing the input across features for each data sample, resulting in \( F_{\text{norm}} \).

\underline{\textit{Gated Feedforward Network.}} {The Gated Feedforward Network (GDFN) achieves refined fusion of multimodal features through dynamic channel weight. Given the input feature $F_{\text{norm}}$, it is first projected into a higher-dimensional space through a pointwise (1×1) convolution. A subsequent depthwise convolution is applied to capture intra-channel dependencies, producing intermediate features $F_{\text{v}}$. The feature $F_{\text{v}}$ is then split along the channel dimension into two parts, $F_{\text{v}_1}$ and $F_{\text{v}_2}$. A gating mechanism is applied as follows: \( F_{w} = \text{GELU}(F_{\text{v}_1}) \odot F_{\text{v}_2}\), where $\odot$ denotes element-wise multiplication. This operation adaptively reweights the feature channels, suppressing modality-specific redundancy while emphasizing informative signals. Finally, a 1×1 convolution reduces dimensionality: \( \mathbf{Z} = \text{Conv}_{1\times1}(F_{w}) \). This process enables GDFN to perform fine-grained modeling and learnable weighting of multimodal features, effectively enhancing the discriminative power and robustness of the representation.}

The fused vector \( Z \) is then passed to the classifier for anomaly detection. The classifier begins with global average pooling to compute the global averages of the input features, reducing their dimensionality and generating a compact representation. This low-dimensional representation is further processed by a fully connected layer, which produces the final anomaly detection results.

\input{model}

\subsubsection{Model construction and training}

{Algorithm 1 outlines the construction and training process of the Multimodal-CNN. }

{\underline{\textit{Construction.}} {The algorithm begins by defining input parameters, including the original time series data (denoted as $\textit{input}^n$), image time series data (denoted as $\textit{input}^i$), model parameters (denoted as $\textit{parameter\_list}$), the weight, bias and mask of each layer (denoted as $W$, $B$, and $M$), the number of layers (denoted as $\textit{layers}$), and the number of training epochs (denoted as $\textit{epoch\_num}$)}.}

{Next, convolutional layers are constructed iteratively. For each layer, the algorithm checks if it is the final layer and determines whether to apply the \texttt{ReLU} activation function accordingly (lines 2–3). The \texttt{Build\_Layer} function is then used to create the current layer, which is appended to the \textit{layer\_list}. Finally, all layers are combined into a sequential CNN model named \textit{CNN-Block} (lines 4–5).}

{\underline{\textit{Training.}} During the training phase, the algorithm iterates through each training epoch. For each batch, it processes the original time series data using the convolutional blocks. The output \( \textit{output}^{n} \) from each block is added to the original input, followed by applying the \texttt{ReLU} function to generate a new input \( \textit{output}^{n} \) (lines 7--11).}

{Second, the algorithm applies a fully connected layer and pooling operation to \( \textit{output}^{n} \), producing the feature \( \textit{F}_{\text{num}} \) (line 12). Similarly, the image time series data is processed, generating \( \textit{output}^{i} \), which passes through a fully connected layer and pooling operation to derive the feature \( \textit{F}_{\text{img}} \) (lines 14--16).}

{Next, {We initialize the parameters \( Q \), \( K \), and \( V \) for the multimodal attention mechanism, which is then employed to capture and fuse inter-modal relationships, producing updated features \( \textit{F}_{\text{num}} \) and \( \textit{F}_{\text{img}} \) (lines 18--21)}. These features are fused into a combined feature \( \textit{F}_{\text{concat}} \), and the model is optimized using the Adam optimizer~\cite{adam} with a learning rate adjustment function (lines 22--23).}

{Finally, the algorithm returns the trained MultiModal Feature Extraction model \textit{model}. This enables effective extraction and fusion of features from multiple modalities, enhancing the model's overall performance.}

\subsection{Anomaly Explainer}
The explainer interprets the results of anomaly detection models, as illustrated in the lower right corner of {Fig.}~\ref{figuresystem}. It includes the Kernel Explainer and Gradient Explainer~\cite{SHAP}, which are used to interpret numerical and image data, respectively. Kernel methods and gradient methods are employed to generate SHAP values (\( s^K \) and \( s^G \)), which evaluate each variable's contribution to anomaly detection. \( s^K \) represents the direct impact of the current variable's raw data on anomaly detection results, while \( s^G \) captures the influence of the current variable's MV-MTF data on the current variable. During anomaly evaluation, \( s^K \) and \( s^G \) are combined to identify the variables that significantly contribute to anomalies. For the \( v^{\textit{th}} \) variable, its SHAP value \( s_v \) is computed as follows:
\begin{equation}
     s_{v} = \omega \times s_{v}^{K} + (1 - \omega) \times \frac{\sum_{j=1}^{c} s_{j}^{K} \times s_{v}^{G}}{c},
     \label{Equation 14}
\end{equation}

\noindent where \( c \) is the total number of variables, \( \omega \) is a weight parameter balancing the contributions of \( s^K \) and \( s^G \), \( s_{v}^{K} \) is the SHAP value reflecting the direct impact of the \( v^{\textit{th}} \) variable, and \( s_{v}^{G} \) is the impact of other variables on the \( v^{\textit{th}} \) variable. Equation~\ref{Equation 14} integrates both direct and indirect contributions to better identify variables most responsible for anomalies. 

{
Let the fused attribution for variable $v$ be $s_v(\omega)=\omega\,s_v^{K}+(1-\omega)\,\gamma\,s_v^{G}$ with $\gamma=\tfrac{1}{c}\sum_{j=1}^c s_j^{K}$. Collecting terms in 
$\omega$ yields $s_v(\omega)=a_v\,\omega+b_v$ with $a_v:=s_v^{K}-\gamma s_v^{G}$ and $b_v:=\gamma s_v^{G}$.}

{\begin{lemma}[Bounded controllability] 
Given $s_v(\omega)=a_v\,\omega+b_v$ and let $L:=\|a\|_\infty=\max_{v}|a_v|$. Then for any $v$ and any $\omega, \omega'\in[0,1]$, $\big|s_v(\omega')-s_v(\omega)\big|\ \le\ |a_v|\,|\omega'-\omega|\ \le\ L\, |\omega'-\omega|.$ Hence each $s_v$ is $L$-Lipschitz in $\omega$, ensuring bounded sensitivity and controllability with respect to the weighting parameter.
\label{lemma1}
\end{lemma}
\vspace{-5mm}
\begin{proof} 
For any perturbation $\delta$ with $\omega+\delta\in[0,1]$, since $s_v(\omega)=a_v\,\omega+b_v$, we have
\begin{equation}
    \begin{aligned}
        s_v(\omega+\delta)-s_v(\omega)&=a_v\,\delta
\quad\\
\Rightarrow\big|s_v(\omega+\delta)-s_v(\omega)\big|&=|a_v|\,|\delta|\le \underbrace{\max_{v}|a_v|}_{=\,\|a\|_\infty=L}\,|\delta|.
    \end{aligned}
\end{equation}
Equivalently, for any $\omega,\omega'\in[0,1]$,
\begin{equation}
    \big|s_v(\omega')-s_v(\omega)\big|=\big|a_v(\omega'-\omega)\big|\le L\,|\omega'-\omega|.
\end{equation}
Hence, each $s_v$ is $L$-Lipschitz in $\omega$, establishing bounded controllability. 
\end{proof}}

{As $L = max_v |a_v^{K}|$ and $a_v = s_v^{K} - \gamma s_v^{G}$, we can get that $L \le max_v|s_v^K| + \gamma \cdot max_v|s_v^G|$. In binary classification, it is known that $\max_v |s_v^{K}|=1$, and for anomalous data the MV-MTF values are very small, which leads to $|s_v^{G}|\approx 0$. Therefore, we can derive that $L \le 1$. Hence, for any $\omega, \omega'$, it holds that $|s_v(\omega') - s_v(\omega)| \le |\omega' - \omega|$. This indicates that as $\omega$ varies, the contribution of variable $v$ is constrained by a uniform upper bound, and its ranking is therefore stable.}

{Given the ranked contribution list $C(\omega) = \{s_{i}(\omega) \}_{i=1}^n$, where $n$ is the number of variables and $C(\omega)$ is  in descending order of $s_i(\omega)$, the Top-$k$ set is defined as $k$ variables with highest $s_i(\omega)$, i..e, $\mathrm{Top}\text{-}k(\omega) = \{s_{1}(\omega), \dots, s_{k}(\omega)\}$.}

{\begin{lemma}[Top-$k$ ranking plateau]  For any two variables $u$ and $v$, the pairwise contribution difference is defined as $\Delta_{uv}(\omega) := s_u(\omega) - s_v(\omega) = A_{uv}\omega + B_{uv}$, where $A_{uv} = a_u - a_v$ and $B_{uv} = b_u - b_v$. 
Then, with $S := \max_{u,v}|A_{uv}|$, the Top-$k$ set remains locally invariant within radius $|\delta| = (s_{k}(\omega) - s_{k+1}(\omega))/S$, i.e., $
\mathrm{Top}\text{-}k(\omega+\delta) \equiv \mathrm{Top}\text{-}k(\omega).$
\label{lemma2}
\end{lemma}
\begin{proof} For any $u\in\mathrm{Top}\text{-}k$ and $v\notin\mathrm{Top}\text{-}k$, given the offset value $\delta$ of $\omega$, then 
\begin{equation}
    \begin{aligned}
        \Delta_{uv}(\omega+\delta) &= A_{uv}(\omega+ \delta) + B_{uv} \\
        &=\Delta_{uv}(\omega)+A_{uv}\,\delta\ \\
        &\ge\ \Delta_{uv}(\omega)-|A_{uv}|\,|\delta|\  \\
        &\ge\ (s_{k}(\omega) - s_{k+1}(\omega)) -|A_{uv}|\,|\delta|.
    \end{aligned}
\end{equation}
Only when $\Delta_{uv}(\omega+\delta) > 0$, we have $\mathrm{Top}\text{-}k(\omega+\delta)=\mathrm{Top}\text{-}k(\omega)$. To ensure this condition simultaneously across all variable pairs, we set $S := \max_{u,v}|A_{uv}|,$ which leads to the sufficient condition $|\delta| < (s_{k}(\omega) - s_{k+1}(\omega))/S$. Here, $s_{k}(\omega) - s_{k+1}(\omega)$ represents the minimal margin between Top-$k$ and non-Top-$k$ variables. Under this condition, the Top-$k$ set remains unchanged. \end{proof}}

{$S:=\max_{u,v}|A_{uv}|$ is the same for all $\omega$. Therefore, $\delta$ is determined by $s_{k}(\omega) - s_{k+1}(\omega)$. From $s_v(\omega)=\omega\,s_v^{K}+(1-\omega)\,\gamma\,s_v^{G},$ we can find that increasing $\omega$ raises the contribution of the raw data $s^K$, while decreasing $\omega$ raises the contribution of the transition MV-MTF data $\gamma s^G$. Under anomalies, the values of the transition structure MV-MTF tend to decrease (due to the low probability of anomalous transitions), which in turn drives the contribution $s^G$ close to zero, thereby leading to a very small $s_{k}(\omega) - s_{k+1}(\omega)$ value. The value of $s_{k}(\omega) - s_{k+1}(\omega)$ increases only when $\omega$ grows, which means the weight of the raw data contribution is increased. According to $|\delta| < \tfrac{s_{k}(\omega) - s_{k+1}(\omega)}{S}$, it follows that as $s_{k}(\omega) - s_{k+1}(\omega)$ increases, $\Delta$ also grows, thereby widening the plateau region. }


This weighting scheme effectively balances the contributions of raw numerical features and MV-MTF modalities, ensuring anomaly detection primarily relies on reliable raw data while moderately integrating the complex structural information captured by MV-MTF.}

Based on Equation~\ref{Equation 14}, we rank variables by their contribution to anomalies. Variables with higher \( s_v \)  contribute more, while those with lower \( s_v \)  have little or no contribution. Identifying the contributing variables is essential. To achieve this, we propose a binary search strategy to efficiently determine a critical point \( k \), defined as follows:

\begin{definition} [\textbf{Critical point}]\label{def:critical_point}
    Given a list of variables \( \mathcal{L} \) ranked by their contribution to anomalies, \( k \) is a critical point such that the first \( k \) variables in \( \mathcal{L} \) contribute to the anomalies, while the remaining variables do not. \( \mathcal{L}_{o, o'} \) denotes a sublist of \( \mathcal{L} \) from \( \mathcal{L}_{o} \) to \( \mathcal{L}_{o'} \), where \( \mathcal{L}_{o} \) is the \( o^{\textit{th}} \) variable in \( \mathcal{L} \).
\end{definition}

\begin{myEx}
For an anomaly time series with eight variables, we first calculate \( s_v \) (\( 1 \leq v \leq 8 \)) using Equation 10 and rank the variables by \( s_v \) to obtain \( \mathcal{L} = \{2, 4, 5, 1, 8, 7, 6, 3\} \). This ranking indicates that variable 2 is ranked first, variable 4 second, and so on.

Next, we apply the binary search strategy. First, we replace the first four variables \( \mathcal{L}_{1,4} = \{2, 4, 5, 1\} \)—the first half of \( \mathcal{L} \)—with normal sequences and check whether the anomaly remains. Two cases arise:  
(i) If the anomaly remains, \( 4 \leq k \leq 8 \), and we replace \( \mathcal{L}_{4, 5} \)—the first half of \( \mathcal{L}_{4, 8} \)—with normal sequences.  
(ii) If the anomaly is resolved, \( 1 \leq k \leq 4 \), and we replace \( \mathcal{L}_{1, 2} \)—the first half of \( \mathcal{L}_{1, 4} \)—with normal sequences.  This process is repeated until \( k \) is determined, identifying the first \( k \) variables as those contributing to the anomalies.
\end{myEx}

During anomaly analysis, anomalous data is clustered and evaluated with expert knowledge to determine the nature of each anomaly. For example, in the SMD dataset, anomalies are categorized into the three types: (i) network failure, (ii) hardware damage, and (iii) software service anomaly. New anomalies can be classified into these predefined types and further interpreted in detail.

Finally, by combining the results of anomaly evaluation and analysis, an anomaly report is generated. This report not only provides detailed explanations of anomaly detection results but also enhances the accuracy and actionability of the analysis through evaluation and classification mechanisms. It offers users deeper insights and robust support for decision-making.


%% file: model.tex
\begin{algorithm}[tb]
\LinesNumbered
\caption{Multimodal-CNN model}
\KwIn{Raw data $\textit{input}^n$, MV\mbox{-}MTF data $\textit{input}^i$, parameter list $\textit{parameter\_list}$; layer weights $W$, biases $B$, masks $M$; number of layers $\textit{layers}$; number of epochs $\textit{epoch\_num}$}
\KwOut{Feature extraction model $\textit{model}$}

\ForEach{$i \in \mathrm{range}(\textit{layers})$}{
  $using\_relu \leftarrow (i == \textit{layers}-1)$\;
  $layer \leftarrow \texttt{Build\_Layer}(layer\_parameters, using\_relu)$\;
  add $layer$ to $layer\_list$\;
}
// initialize $\textit{CNN\mbox{-}Block}$ using $layer\_list$\;

\For{$epoch \gets 1$ \KwTo $\textit{epoch\_num}$}{
  \For{$j \gets 1$ \KwTo $\textit{batch\_num}$}{
    \ForEach{$\text{CNN} \in \textit{CNN\mbox{-}Block}$}{
      $\textit{output}^n \leftarrow \texttt{CNN}(\textit{input}^n, (W,B,M))$\;
      $\textit{output}^{n\prime} \leftarrow \texttt{ReLU}\!\big(\texttt{add}(\textit{input}^n, \textit{output}^n)\big)$\;
      $\textit{input}^n \leftarrow \textit{output}^n$\;
    }
    $\textit{F}_{\mathrm{num}} \leftarrow \texttt{FullConnect}(\texttt{Pooling}(\textit{output}^{n\prime}))$\;

    \For{$t \gets 1$ \KwTo $\textit{CNN\mbox{-}Block}\_{\mathrm{num}}$}{
      $\textit{output}^i \leftarrow \texttt{CNN}(\textit{input}^i, (W,B,M))$\;
      $\textit{output}^{i\prime} \leftarrow \texttt{ReLU}\!\big(\texttt{add}(\textit{input}^i, \textit{output}^i)\big)$\;
      $\textit{input}^i \leftarrow \textit{output}^i$\;
    }
    $\textit{F}_{\mathrm{img}} \leftarrow \texttt{FullConnect}(\texttt{Pooling}(\textit{output}^{i\prime}))$\;

    // initialize $Q$, $K$, $V$ for multimodal attention\;
    $\textit{F}_{\mathrm{num}\to \mathrm{img}} \leftarrow \texttt{Attention}_{\mathrm{num}\to \mathrm{img}} \oplus \textit{V}_{\mathrm{img}}$\;
    $\textit{F}_{\mathrm{img}\to \mathrm{num}} \leftarrow \texttt{Attention}_{\mathrm{img}\to \mathrm{num}} \oplus \textit{V}_{\mathrm{num}}$\;
    $\textit{F}_{\mathrm{concat}} \leftarrow \texttt{Fusion}(\textit{F}_{\mathrm{num}\to \mathrm{img}}, \textit{F}_{\mathrm{img}\to \mathrm{num}})$\;

    \texttt{Adam}()\;
  }
}
\Return{$model$}
\end{algorithm}

%% file: experiment.tex
\section{Experiments} 

\subsection{Experimental Setting}
\noindent\textbf{Datasets.} We evaluate \textsc{Moon} on six real-world MTS datasets.
\begin{itemize}
[topsep=1pt,itemsep=1pt,parsep=1pt, partopsep=1pt]
    \item  Soil Moisture Active Passive (SMAP)~\cite{SMAP}: It consists of soil samples and telemetry data collected by NASA's Mars rover.
    \item  Mars Science Laboratory (MSL)~\cite{MSL}: It is similar to SMAP but corresponds to the sensor and actuator data for the Mars rover itself.
    \item  Secure Water Treatment (SWaT)~\cite{SWaT}: It is collected from a real-world water treatment plant with seven days of normal and four days of abnormal operation. This dataset consists of sensor values (water level, flow rate, etc.) and actuator operations (valves and pumps).
    \item Water Distribution dataset (WADI)~\cite{SWaT}: It is an extended dataset of SWAT. WADI features 14 days of normal operation KPIs and two days of attack scenario KPIs.
    \item  Pooled Server Metrics (PSM)~\cite{PSM}: It is collected from multiple eBay application server nodes, anonymized for publication. It includes 26 features related to server machine metrics like CPU utilization and memory, with some localization meta-attributes omitted. The training set spans 13 weeks, followed by eight weeks for testing.
    \item Server Machine Dataset (SMD)~\cite{SMD}: It is a five-week long dataset of stacked traces of the resource utilizations of 28 machines from a compute cluster. Similar to MSL, we use the nontrivial sequences, specifically the traces for machine-2-1, machine-2-6, and machine-3-7.
\end{itemize}

The datasets used are publicly available and can be accessed at the following URLs: SMAP\footnote{\protect\url{https://nsidc.org/data/smap/data}}, MSL\protect\footnote{\url{https://pds-atmospheres.nmsu.edu/data_and_services/atmospheres_data/Mars/Mars.html}}, SWaT\footnote{\protect\url{https://itrust.sutd.edu.sg/itrust-labs_datasets/dataset_info/\#swat}}, WADI\footnote{\protect\url{https://itrust.sutd.edu.sg/itrust-labs_datasets/dataset_info/\#wadi}}, PSM\footnote{\protect\url{https://github.com/eBay/RANSynCoders/tree/main/data}}, and SMD\footnote{\protect\url{https://github.com/NetManAIOps/OmniAnomaly/tree/master/ServerMachineDataset}}. We summarize the key characteristics of the datasets in Table~\ref{dataset}.  ``\#Entities'' refers to the number of distinct time series, and ``\#Dimensions'' refers to the number of dimensions in each dataset. 
Each dataset includes a training set without anomaly detection labels and a testing set with labels. It is important to note that since our task is supervised, we only use the test portion of the datasets. For our experiments, we use 6,000 of testing set for training for the WADI dataset and 10,000 for the other five datasets. All datasets use the remaining 5,000 labeled samples for testing. 

\begin{table}[t]
\centering
\caption{Dataset Statistics}
\footnotesize
\setlength{\tabcolsep}{2.5mm}
\begin{tabular}{c|cccc}
\hline
DataSet & \#Train   & \#Test   & \#Entities & \#Dimensions \\ \hline
WADI    & 1048571 & 172801 & 1        & 123        \\
MSL     & 58317   & 73729  & 3        & 55         \\
PSM     & 132482  & 87842  & 1        & 26         \\
SMAP    & 58317   & 73729  & 55       & 25         \\
SWaT    & 496800  & 449919 & 1        & 51         \\
SMD     & 708405  & 708420 & 28       & 38         \\ \hline 
\end{tabular}
\label{dataset}
\end{table}

\begin{table*}[t]
\centering
\caption{{Comparison results}}  \small
\setlength{\tabcolsep}{1.8mm}
\begin{tabular}{ccccccccccccc}
\hline
\multirow{2}{*}{Method} & \multicolumn{4}{c}{WADI}                                              & \multicolumn{4}{c}{MSL}                                               & \multicolumn{4}{c}{PSM}                                               \\ \cline{2-13} 
                        & Precision       & Recall          & F1              & F1PA            & Precision       & Recall          & F1              & F1PA            & Precision       & Recall          & F1              & F1PA            \\ \hline
TranAD                  & 0.9460          & 0.9992          & 0.9719          & 0.9742          & 0.8649          & 0.9831          & 0.9202          & 0.9914          & 0.9494          & 0.9999          & 0.9494          & 0.9784          \\
OmniAnomaly             & 0.8549          & 0.9999          & 0.9218          & 0.9478          & 0.7848          & \textbf{0.9999} & 0.8794          & 0.9917          & 0.8816          & 0.9990          & 0.8985          & 0.9522          \\
MAD\_GAN                & 0.9422          & 0.9596          & 0.9702          & 0.9702          & 0.8516          & \textbf{0.9999} & 0.9198          & 0.9801          & 0.8725          & 0.9968          & 0.9287          & 0.9417          \\
LSTM\_AD                & 0.8953          & \textbf{0.9999} & 0.9447          & 0.9686          & 0.7948          & \textbf{0.9999} & 0.9023          & 0.9683          & 0.9038          & 0.9999          & 0.9494          & 0.9814          \\
TimesNet                & 0.7973          & 0.9900          & 0.8833          & 0.9948          & 0.9735          & 0.9715          & 0.9725          & 0.9855          & 0.9850          & 1.0000          & 0.9924          & 1.0000          \\
ADtransformer           & 0.7926          & 0.9620          & 0.8692          & 0.9804          & \textbf{0.9802}          & 0.9842          & 0.9822          & 0.9920          & 0.9908          & 0.9878          & 0.9893          & 0.9939          \\
DCdetector              & 0.7961          & 0.9144          & 0.8512          & 0.9551          & 0.9739          & 0.9589          & 0.9664          & 0.9790          & 0.9809          & 0.8910          & 0.9338          & 0.9358          \\
CATCH                   & 0.7968          & 0.9524          & 0.8677          & 0.9754          & 0.9773   & 0.8651          & 0.9178  & 0.9277          & 0.9815           & 0.9745           & 0.9780           & 0.9780           \\
HyperRocket             & -               & -               & -               & -               & 0.9771          & 0.9355          & 0.9558          & 0.9691          & 0.9820          & 1.0000          & 0.9909          & 0.9909          \\
Rocket                  & 0.6923          & 0.8735          & 0.9945          & 0.9945          & 0.9773          & 0.9378          & 0.9571          & 0.9774          & 0.9850          & 1.0000          & 0.9909          & 0.9909          \\
Moon                    & \textbf{1.0000} & 0.9912          & \textbf{0.9956} & \textbf{0.9956} & 0.9772          & 0.9961          & \textbf{0.9866}          & \textbf{0.9980} & \textbf{0.9869} & \textbf{1.0000} & \textbf{0.9934} & \textbf{1.0000} \\ \hline
\multirow{2}{*}{Method} & \multicolumn{4}{c}{SMAP}                                              & \multicolumn{4}{c}{SWaT}                                              & \multicolumn{4}{c}{SMD}                                               \\ \cline{2-13} 
                        & Precision       & Recall          & F1              & F1PA            & Precision       & Recall          & F1              & F1PA            & Precision       & Recall          & F1              & F1PA            \\ \hline
TranAD                  & 0.9133          & 0.9965          & 0.9531          & 0.9706          & \textbf{0.9977}          & 0.6878          & 0.8143          & 0.9773          & 0.9072          & 0.9973          & 0.9501          & 0.9981          \\
OmniAnomaly             & 0.7991          & 0.9989          & 0.8883          & 0.9619          & 0.9760          & 0.6956          & 0.8123          & 0.9769          & 0.9881          & 0.9985          & 0.9932          & 0.9983          \\
MAD\_GAN                & 0.8157          & 0.9999          & 0.8984          & 0.9634          & 0.9593          & 0.6956          & 0.8064          & 0.9770          & \textbf{0.9967} & 0.9980          & \textbf{0.9973} & 0.9982          \\
LSTM\_AD                & 0.8139          & 0.9999          & 0.8974          & 0.9702          & 0.9977          & 0.6878          & 0.8143          & 0.9318          & 0.9069          & 0.9973          & 0.9499          & 0.9695          \\
TimesNet                & 0.9107          & 0.9993          & 0.9530          & 0.9997          & 0.8072          & 0.9988          & 0.8928          & 0.9994          & 0.9918          & 0.9989          & 0.9953          & \textbf{0.9994} \\
ADtransformer           & 0.9103          & 0.9939          & 0.9502          & 0.9969          & 0.8074          & \textbf{1.0000}          & \textbf{0.8934}          & \textbf{1.0000}          & 0.9930          & 0.9951          & 0.9940          & 0.9944          \\
DCdetector              & 0.9107          & 0.9405          & 0.9253          & 0.9693          & 0.8075          & 0.9978          & 0.8926          & 0.9989          & 0.9918          & 0.8554          & 0.9185          & 0.9219          \\
CATCH                   & 0.9109          & 0.9785          & 0.9435          & 0.9891          & 0.8077    & 0.9762           & 0.8840   & 0.9880           & 0.9916           & 0.9357             & 0.9628            & 0.9661           \\
HyperRocket             & 0.9107          & 0.9989          & 0.9528          & 0.9961          & 0.7379          & 0.0907          & 0.1615          & 0.1663          & 0.9933          & 0.9311          & 0.9612          & 0.9612          \\
Rocket                  & 0.9106          & 1.0000          & 0.9532          & 0.9532          & 0.6936          & 0.4985          & 0.5801          & 0.6653          & 0.9936          & 0.9652          & 0.9792          & 0.9957          \\
Moon                    & \textbf{0.9561}          & \textbf{1.0000} & \textbf{0.9776}          & \textbf{1.0000} & 0.8111          & 0.9787          & 0.8870          & 0.9892          & 0.9914          & \textbf{0.9999} & 0.9956          & 0.9956          \\ \hline

\end{tabular}
\label{results}
\end{table*}

\noindent\textbf{Baselines and setting.} We compare \textsc{Moon} with six state-of-the-art reconstruction-based methods: \textsc{TranAD}~\cite{tranAD}, \textsc{OmniAnomaly}~\cite{OmniAnomaly}, \textsc{MAD\_GAN}~\cite{MAD-GAN}, \textsc{LSTM\_AD}~\cite{DLSTMAD}, and \textsc{CATCH}~\cite{CATCH} and two classification-based methods: \textsc{HyperRocket}~\cite{hydra-MultiRocket23} and \textsc{Rocket}~\cite{Rocket20,MultiRocket22}. We choose the two types of methods as baselines because (i) prediction-based methods are theoretically similar to reconstruction-based methods, with the latter being more mainstream, and (ii) \textsc{Moon} is a classification-based method. Due to the space limitation, please refer to Section \ref{Related Work} for more details about baselines. {In our experiments, the hyperparameters are configured as follows: the coefficient $\alpha$ in Equation \ref{Equation 7} is set to 0.9, the learning rate is set to 0.001, and the weight $\omega$ in Equation \ref{Equation 14} is fixed at 0.6.
}

\noindent\textbf{Performance metrics.} We evaluate anomaly detection performance using Precision, Recall, F1 score, and $\text{F1}_\text{PA}$ score.  
$\text{F1}_\text{PA}$ score is a segment-based F1 score using point adjustment (PA), where a segment is considered abnormal if at least one point within it is detected as abnormal~\cite{PAK,PA}.  $\text{F1}_{\text{PA}} = 2 \cdot \frac{\text{Precision}_{\text{PA}} \cdot \text{Recall}_{\text{PA}}}{\text{Precision}_{\text{PA}} + \text{Recall}_{\text{PA}}}$, where $\text{Precision}_{\text{PA}}$ and $\text{Recall}_{\text{PA}}$ are the precision and recall calculated using the PA method.

We evaluate the interpretability of anomaly detection using HitP\%, which measures the proportion of true anomalous dimensions included among the top candidates predicted by the model~\cite{HitH}. Here, P\% is the percentage of true anomalous dimensions at each timestamp, defining the number of top predicted candidates considered.

\noindent\textbf{Environment.}
All methods are executed on a machine with an Intel(R) Core(TM) i9-10900K CPU, featuring 10 cores and a 3.70GHz clock speed. The framework is also equipped with an NVIDIA GeForce RTX 3090 graphics card, which has 24GB of video memory. The source code of \textsc{Moon} are available\footnote{https://github.com/Syh517/Moon/tree/master}.


\begin{figure}[t]
      \centering
      \includegraphics[width=\linewidth]{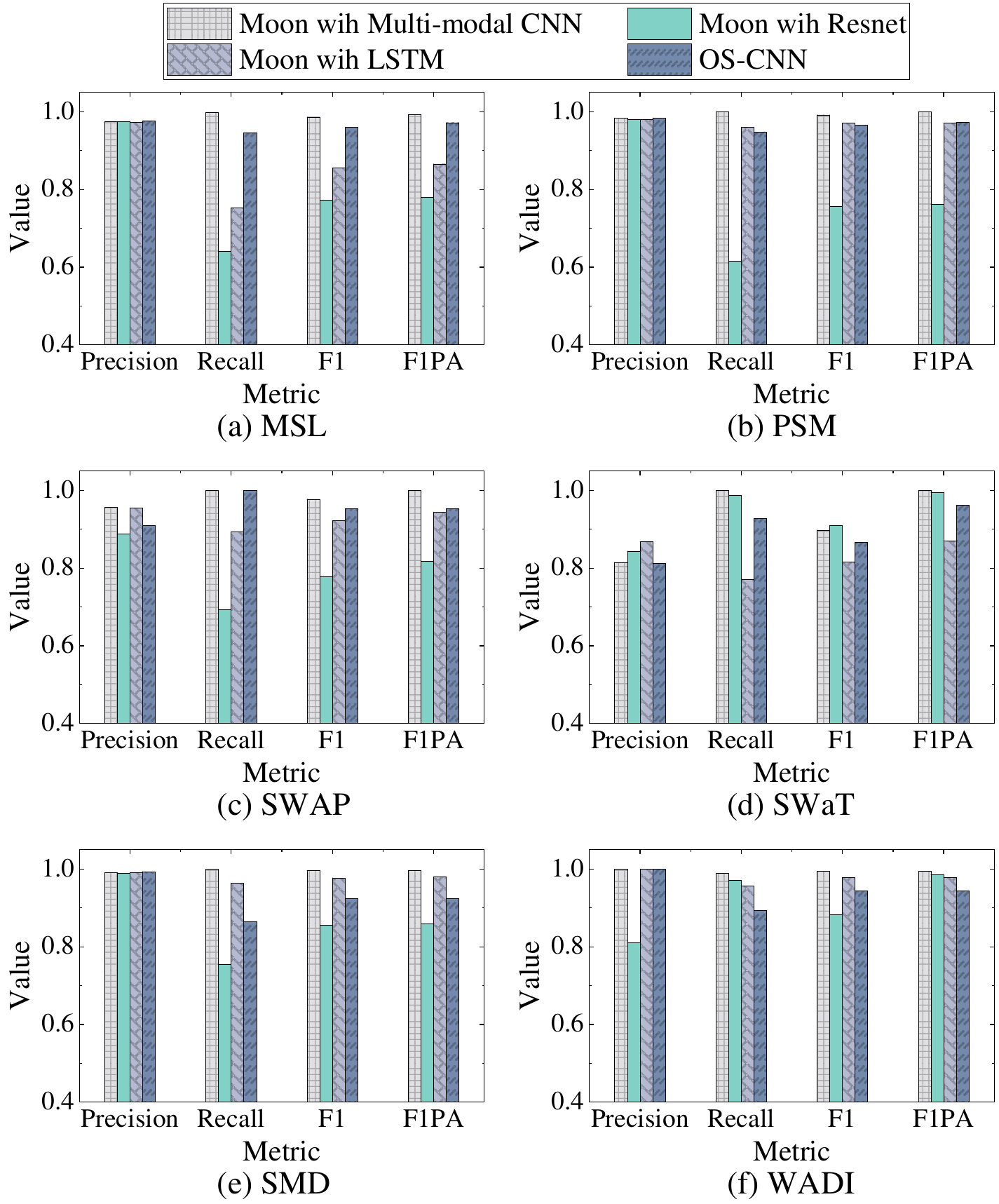}
      \caption{Ablation Study on Multimodal-CNN of Moon} 
      \label{figureablation_cnn}
      \vspace{-3mm}
\end{figure}

\begin{figure}[th]
      \centering
      \includegraphics[width=\linewidth]{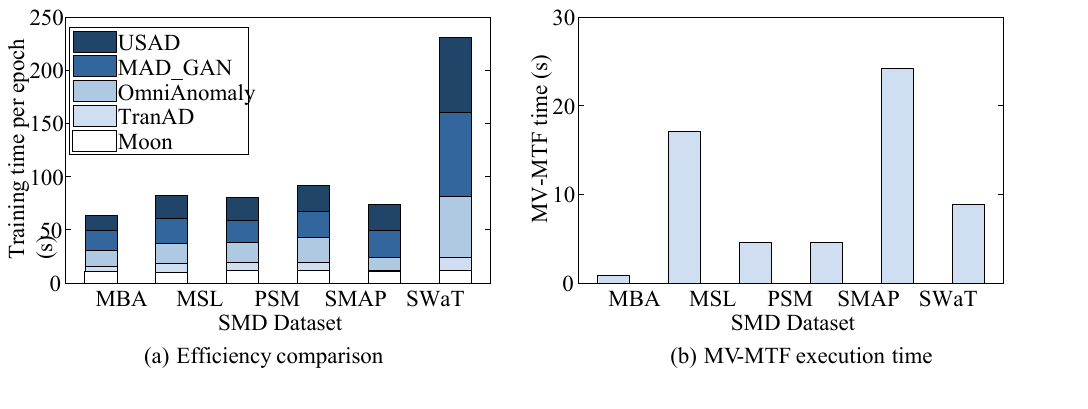}
      \vspace{-6mm}
      \caption{Parameter study}       
      \label{figureefficiency}
      \vspace{-6mm}
\end{figure}

\subsection{Comparison Study}
We compare \textsc{Moon} with reconstruction-based methods (\textsc{TranAD}, \textsc{OmniAnomaly}, MAD\_GAN, LSTM\_AD) and classification-based methods (\textsc{HyperRocket} and \textsc{Rocket}) in six datasets. Table~\ref{results} presents the results, highlighting the best performance in bold.

\noindent\textbf{Compared with reconstruction-based methods.} \textsc{Moon} achieves the highest or near-highest Recall, F1 and $\text{F1}_{\text{PA}}$ on most datasets, with Precision values almost reaching 1 across all datasets. In contrast, reconstruction-based methods such as \textsc{TranAD}, \textsc{OmniAnomaly}, \textsc{MAD\_GAN}, \textsc{CATCH}, and \textsc{ LSTM\_AD} often have lower F1 and $\text{F1}_{\text{PA}}$ on complex datasets, particularly on the SMAP and SWaT datasets where the F1 are below 0.9. This is because the Moon method improves the ability to identify anomalies in complex multivariate time-series data through modal transformation and deeper feature extraction and fusion techniques.  Meanwhile, reconstruction-based methods rely on reconstruction error, which is effective for simple patterns. However, they struggle with multivariate, high-dimensional, and complex dependencies because they cannot fully capture anomaly characteristics, resulting in lower detection performance.


Furthermore, although many reconstruction-based methods generally excel in Recall, detecting most anomalies and thus maintaining a low miss rate, they struggle with Precision, with higher false positive rates. For instance, \textsc{OmniAnomaly} has a Precision of only 0.7991 on the SMAP dataset, and \textsc{MAD\_GAN} has a Precision of only 0.8725 on the PSM dataset, both significantly lower than  \textsc{Moon}. This indicates that these methods  have difficulties distinguishing between normal and anomalous data, possibly due to reconstruction error inadequately capturing anomaly characteristics. In contrast, the comprehensive feature extraction and classification techniques of \textsc{Moon} enable accurate anomaly detection across various datasets. In addition, {\textsc{Moon} outperforms \textsc{CATCH} in most cases. This is because \textsc{CATCH} transforms time series into the frequency domain, capturing frequency features but potentially losing temporal structure, making it less effective for anomalies that rely on fine-grained temporal context. However, \textsc{Moon} uses raw and enhanced anomaly data to better capture details, achieving more accurate detection.}

\noindent\textbf{Compared with classification-based methods.} \textsc{Moon} outperforms the classification-based methods in F1-Score on most datasets in all performance metrics. However, classification-based methods often show a clear disparity between Precision and Recall. For instance, in the SMAP dataset, \textsc{HyperRocket} has a Precision of only 0.9107, while its Recall reaches 1.0. This result indicates that \textsc{HyperRocket} is highly sensitive to anomalies but lacks the ability to accurately distinguish between normal and anomalous data. This occurs because the Rocket method assigns random kernels to each variable, disregarding the correlations between variables. In contrast, Moon enhances the ability to identify anomalies in complex MTS data by capturing local information between variables through modality conversion, showcasing excellent overall performance. {Additionally, HyperRocket performs slightly worse than Rocket. This is because HyperRocket emphasizes global modeling and cross-variable dynamics, which benefits sequence-level classification~\cite{TSec} but may overlook subtle point-wise anomalies. In contrast, Rocket uses random convolutional kernels to extract local features, making it more effective for detecting anomalies at individual time points.} 

Classification-based methods generally achieve higher F1-Score than reconstruction-based methods on most datasets. This is due to classification methods not relying on anomaly scores,  allowing them to more effectively capture abnormal patterns in the data.


\subsection{Training Efficiency Study}
To evaluate the efficiency, we compare the running time per epoch of \textsc{Moon} with that of reconstruction-based methods. \textsc{HyperRocket} and \textsc{Rocket}, which use predefined feature extraction instead of deep neural networks, do not require training. Hence, their epoch running times cannot be measured and are excluded from the comparison.

The results shown in the {Fig.}~\ref{figureefficiency}(a) indicate that \textsc{Moon}'s running time per epoch is comparable to that of \textsc{TranAD} and significantly lower than that of other methods, especially on SMD datasets.
This result stems primarily from two factors. First, \textsc{Moon}’s efficiency is primarily due to its use of parameter sharing, which reduces redundant computations, and its CNN-based feature extraction module, which accelerates the feature extraction process and further enhances overall efficiency. Second, \textsc{TranAD} is a lightweight method which contributes to its fast performance. In contrast, other reconstruction-based methods process all sequential windows, making them more computationally intensive.   

{Fig.}~\ref{figureefficiency}(b) presents the execution time of MV-MTF technology in \textsc{Moon} across all six datasets.  The results demonstrate that the execution time of MV-MTF increases with the number of variables. For instance, it takes only a few seconds when the number of variables is less than 40, as seen in PSM (26 variables), SMAP (25 variables), and SMD (38 variables). However, the execution time increases significantly when the number of variables exceeds 50, as demonstrated by the MSL (55 variables), SWaT (51 variables) and WADI (123 variables) datasets. However, MV-MTF transformation is performed once during the entire training process. We have optimized its time complexity to \(O(n)\) (see Section 4.1), making its time consumption negligible compared to overall training time. In real-time applications, reconstruction- and prediction-based methods compare with original data, while Moon uses modal conversion. Both have \(O(n)\) complexity, proportional to window size and number of variables, resulting in minimal efficiency differences. 


\begin{figure}[t]
      \centering
      \includegraphics[width=0.7\linewidth]{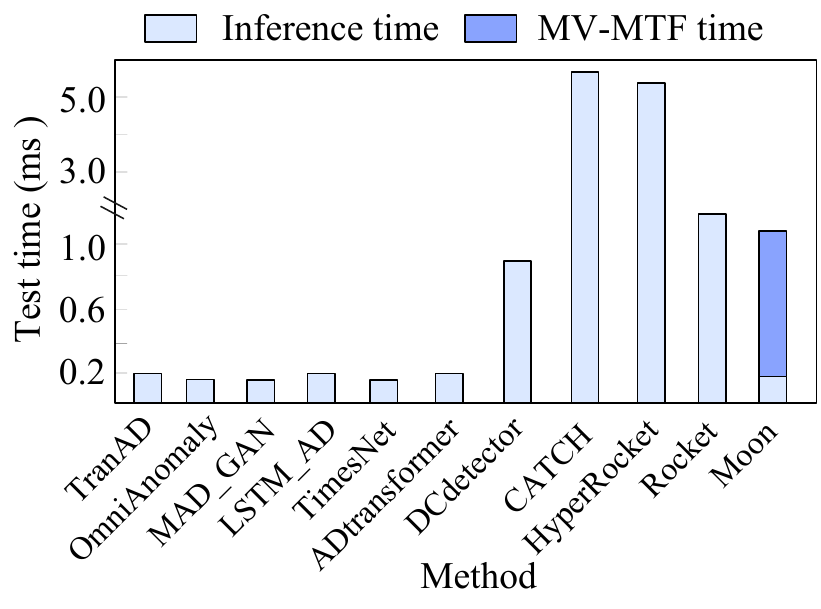}
      \vspace{-3mm}
      \caption{The test time (ms) of each time series data point}
      \vspace{-4mm}
      \label{figureonline-efficiency}
     
\end{figure}

\subsection{Inference Efficiency Study}
{To demonstrate the practical feasibility of our approach in real-time streaming settings, we further report the online testing time, as shown in {Fig.}~\ref{figureonline-efficiency}. It is important to note that our method is not designed for real-time anomaly detection in streaming settings. Therefore, we first calculate the total test time, then divide by the number of samples to get the inference time per TS point. Experimental results show that the anomaly detection time (i.e., 0.2ms) of our classification-based model is comparable to that of existing unsupervised methods. Although the MV-MTF transformation introduces an additional 1.5 ms per time series data point, the overall time remains more efficient than traditional supervised classifiers Rocket and HyperRocket, as our model avoids complex high-dimensional computations during inference.}

\begin{table}[t]
\centering
\caption{Ablation study on modality conversion, feature fusion, and parameter sharing of \textsc{Moon}} \footnotesize
\setlength{\tabcolsep}{1.7mm}
\begin{tabular}{lcccc}
\hline
\multicolumn{1}{c}{\multirow{2}{*}{Method}} & \multicolumn{4}{c}{PSM}                                                             \\ \cline{2-5} 
\multicolumn{1}{c}{}                        & \multicolumn{1}{c}{Precision} & \multicolumn{1}{c}{Recall} & \multicolumn{1}{c}{F1} & \multicolumn{1}{c}{$\text{F1}_\text{PA}$}\\ \hline
Moon                                        & \textbf{0.9869}                         & \textbf{1.0000}                      & \textbf{0.9934}   & \textbf{1.0000}            \\
w/o modality conversion     &  	                  &                         &                  &               \\
---- {only raw data}     &  0.9829	                  & 0.9470	                        & 0.9652                   & 0.9723               \\
---- {only MV-MTF data}     &  0.9819	                  & 0.9979	                        & 0.9898                   & 0.9898               \\
---- {replace MV-MTF with GAF}     &  0.9790		                & 0.8584	                        & 0.9148                   & 0.9148             \\
w/o feature fusion           & 0.9817                        & 0.9837                     & 0.9827         & 0.9858        \\
w/o    parameter sharing      		                & 0.9812                        & 0.9613                     & 0.9711     & 0.9711    \\
w/o    {MultiModel attention}      		                & 0.9806                       & 0.8268                     & 0.8972     & 0.9022       

\\ \hline
\multicolumn{1}{c}{\multirow{2}{*}{Method}} & \multicolumn{3}{c}{SMD}                                                             \\ \cline{2-5}
\multicolumn{1}{c}{}                        & \multicolumn{1}{c}{Precision} & \multicolumn{1}{c}{Recall} & \multicolumn{1}{c}{F1} & \multicolumn{1}{c}{$\text{F1}_\text{PA}$} \\ \hline
Moon                                        & \textbf{0.9914}                        & \textbf{0.9999}                     & \textbf{0.9956}      & \textbf{0.9956}             \\
w/o modality conversion       	               &                         &                      &        &            \\
---- {only raw data}     &  0.9829	                  & 0.9470	                        & 0.9652                   & 0.9723               \\
---- {only MV-MTF data}     &  0.9905	                  & 0.9968	                        & 0.9942                   & 0.9965               \\
---- {replace MV-MTF with GAF}     &  0.9910		                & 0.7568	                        & 0.8597                 & 0.8788             \\
w/o feature fusion        & 0.9912                        & 	0.6131	                   & 0.7578           & 0.7602        \\
w/o parameter sharing                 & 0.9916                        & 0.7822                     & 0.8745         & 0.9956  \\
w/o    {MultiModel attention}      		                & 0.9932                        & 0.7554                     & 0.8582     & 0.8582             

\\ \hline
\end{tabular}
\label{ablation}
\end{table}

\subsection{Ablation Study}

\noindent \textbf{Ablation study on multimodal-CNN.} We evaluate the effectiveness of the proposed Multimodal-CNN by comparing it to ResNet, LSTM, and CNN models across four metrics: Precision, Recall, F1, and $\text{F1}_\text{PA}$, on six datasets. {Fig.}~\ref{figureablation_cnn} presents the comparative results, demonstrating that Multimodal-CNN consistently outperforms both ResNet, LSTM and CNN across all evaluation metrics, with especially notable improvements over ResNet. This arises from Multimodal-CNN's use of convolutional kernels with varying sizes, enabling it to capture receptive fields at different scales and model both local and global features effectively. In contrast, ResNet's fixed receptive field limits its ability to fully capture the range of data features.

Furthermore, Multimodal-CNN demonstrates a clear advantage over LSTM and CNN in handling high-dimensional MTS data, such as MSL with 55 variables and WADI with 123 variables. This advantage stems from its ability to effectively capture complex inter-variable relationships using a MV-MTF and to enhance key information extraction through a cross-modal attention mechanism. These features enable Multimodal-CNN to better integrate multimodal information and model dependencies between variables. In contrast, LSTM and CNN lack the capability to effectively capture inter-variable dependencies. As a result, they often overlook interactions and relationships between variables, leading to the loss of critical information.

\noindent \textbf{Ablation study on four key components.} We evaluate the effectiveness of three key components—modality conversion, feature fusion, parameter sharing, and multimodal attention—through ablation experiments on the PSM and SMD datasets. Table~\ref{ablation} presents the results, with bold values indicating the best performance. 

As observed, omitting any component leads to a performance drop, particularly in Recall, F1, and \(\text{F1}_\text{PA}\). On the PSM dataset, removing the feature fusion module reduces Recall, F1, and \(\text{F1}_\text{PA}\) by 0.0530, 0.0282, and 0.0277, respectively. This is because simple data concatenation fails to capture the relationships between modalities, resulting in decreased performance. On the SMD dataset, omitting modality transformation (i.e., only raw data) causes a significant drop in Recall, F1, and \(\text{F1}_\text{PA}\), decreasing by 0.1349, 0.0688, and 0.0690, respectively. This highlights the critical role of modality transformation in capturing local inter-variable relationships and improving performance. {Only the MV-MTF data or replacing MV-MTF with Gramian Angular Fields (GAF) method still lead to a decline in performance, particularly on the PSM dataset. This is because MV-MTF is specifically designed to capture temporal dependencies and inter-variable relationships in time series. This enables more effective modeling of complex dynamic behaviors and significantly enhances the distinguishability between normal and anomalous patterns. In contrast, GAF primarily relies on static spatial or temporal mappings and struggles to uncover such rich structural information.} In addition, the absence of parameter sharing leads to a notable performance decline, emphasizing its importance in integrating information across modalities to enhance feature extraction and overall model performance. {Finally, the ablation of the multimodal attention module results in a substantial performance degradation, with Recall, F1, and \(\text{F1}_\text{PA}\) dropping to 0.2444, 0.1374, and 0.1374, respectively. This highlights the pivotal role of multimodal attention in capturing cross-modal correlations and complementary information, thereby improving the model's ability to detect abnormal patterns.
}

\begin{figure}[t]
      \centering
      \includegraphics[width=\linewidth]{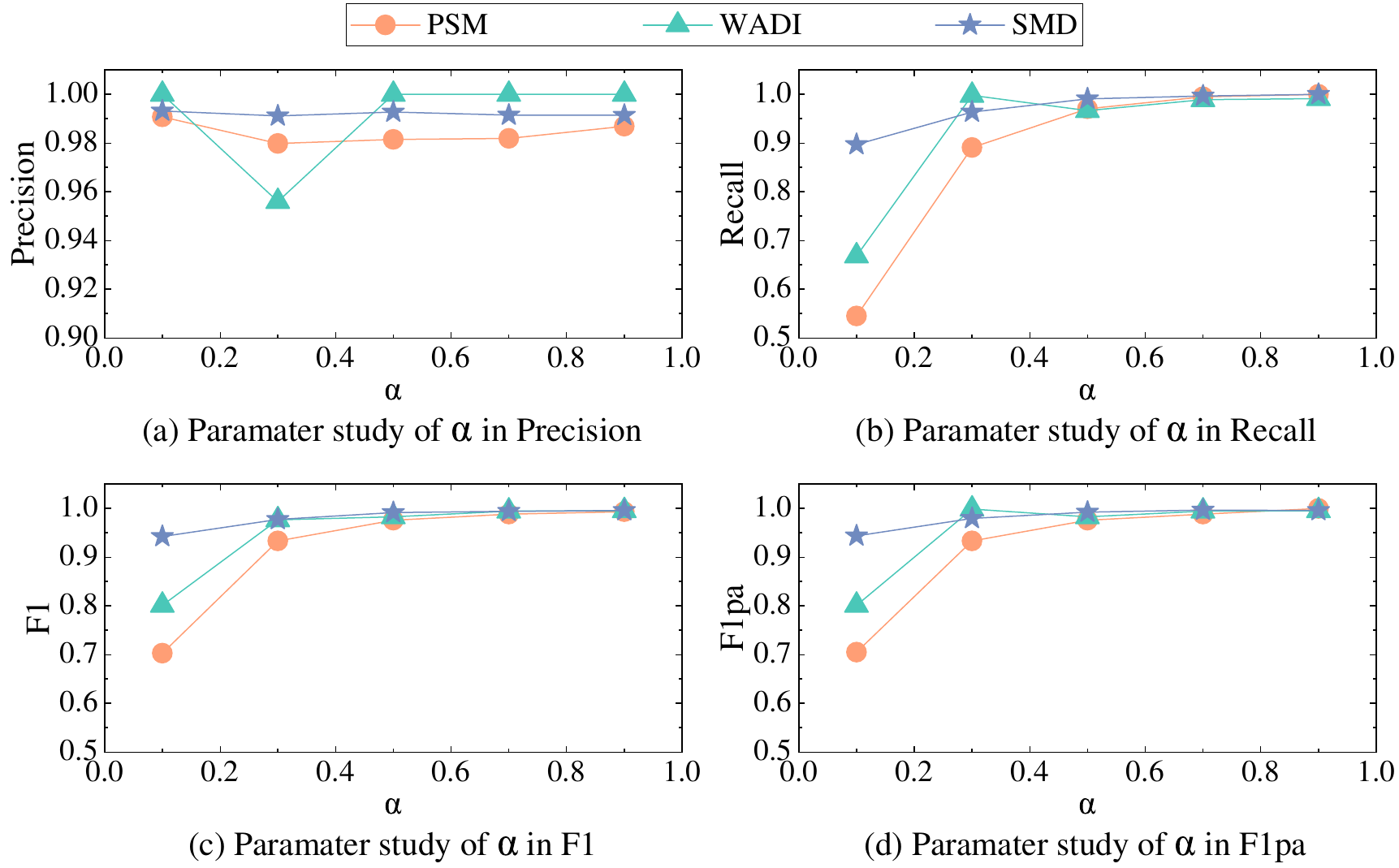}
      \vspace{-3mm}
      \caption{Parameter study of $\alpha$} 
      \label{figureparameter_study}
      \vspace{-3mm}
\end{figure}

\begin{figure}[t]
      \centering
      \includegraphics[width=\linewidth]{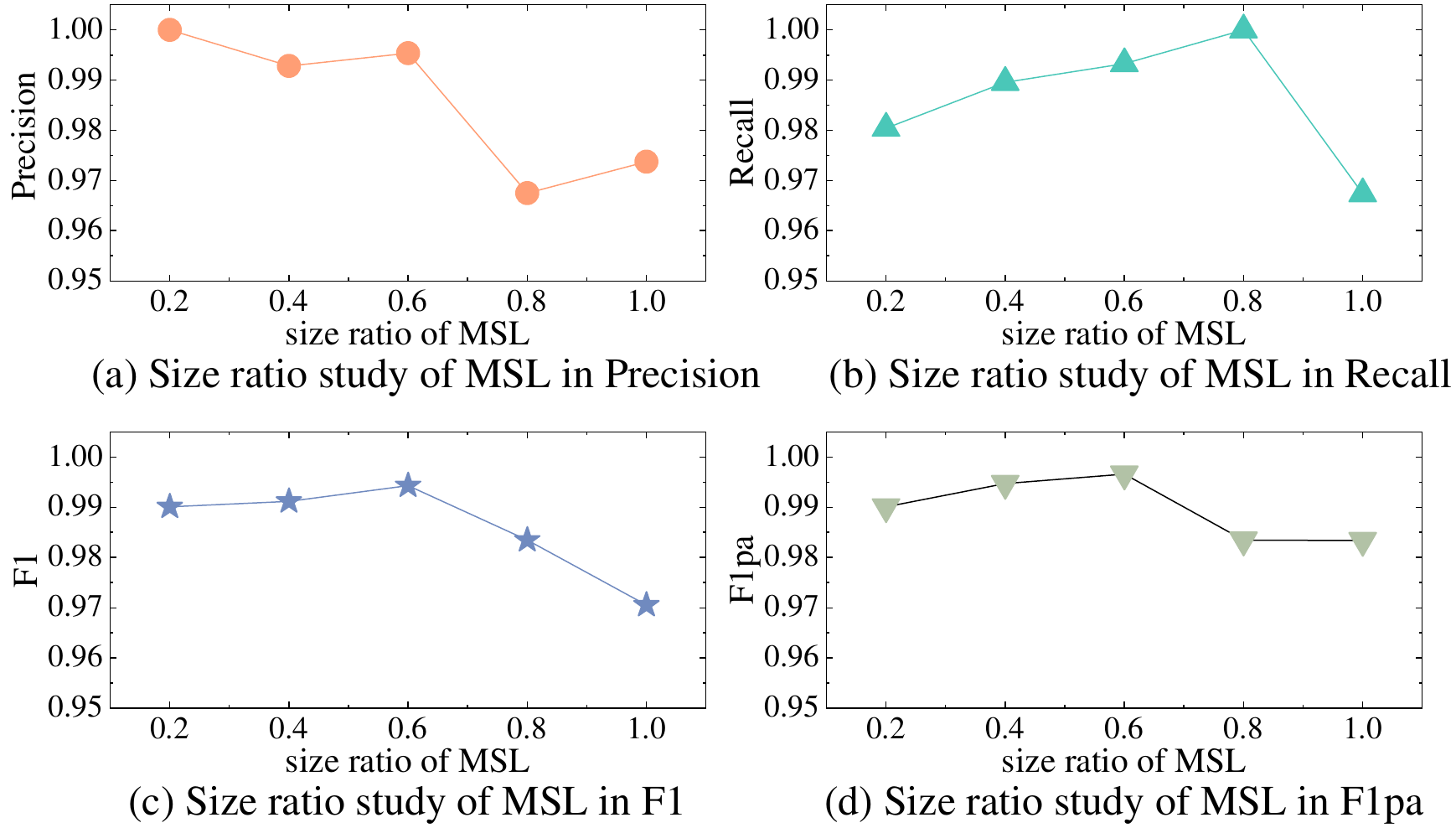}
      \vspace{-3mm}
      \caption{Parameter study on historical time-step length in performance} 
      \label{figurews}
      \vspace{-3mm}
\end{figure}

\begin{figure}[t]
      \centering
      \includegraphics[width=0.6\linewidth]{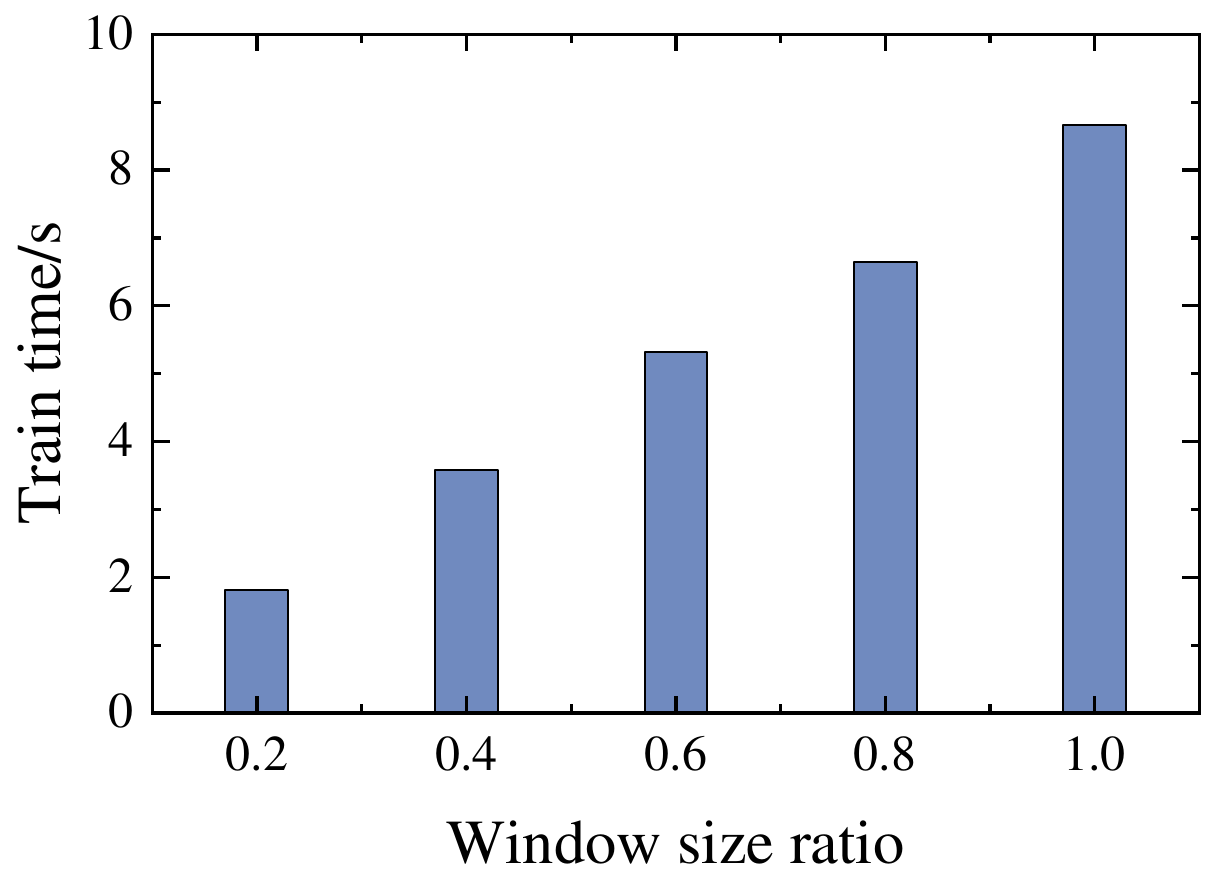}
      \vspace{-3mm}
      \caption{{The efficiency of different window size} } 
      \label{figurews_efficiency}
      \vspace{-7mm}
\end{figure}

\subsection{Parameter Study}
\textbf{Parameter study on $\alpha$.} We conduct a parameter study on the weight parameter $\alpha$ (cf. Equation 5), using the PSM, SMAP, and SMD datasets. As shown in {Fig.}~\ref{figureparameter_study}, $\alpha$ exhibits a similar trend. When $\alpha$ is set to 0.1, the performance is notably poor because this setting reduces the transition probabilities between timestamps, leading to insufficient information capture from individual time dimensions and consequently, lower performance. As $\alpha$ increases, performance generally improves. However, the optimal $\alpha$ value for achieving the best performance varies across datasets due to differences in variable dependencies. 

\textbf{Parameter study on historical window sizes.} {In addition, we conduct a sensitivity analysis on different historical window sizes on MSL dataset to evaluate the robustness of our method. The results presented in {Fig.}~\ref{figurews} demonstrate that the model achieves optimal performance—in terms of bothF1 and $F1_\text{PA}$ scores—when the window size is set to 0.6. This is because a smaller window may fail to capture the distinction between normal and abnormal patterns, while a larger window tends to obscure critical local anomaly signals. From an efficiency perspective, smaller window sizes result in faster computation, as illustrated in {Fig.}~\ref{figurews_efficiency}. This is because both the transformation overhead of MV-MTF and the overall model complexity increase with the growth of window size.}

\begin{figure*}
    \centering
    \includegraphics[width=0.85\linewidth]{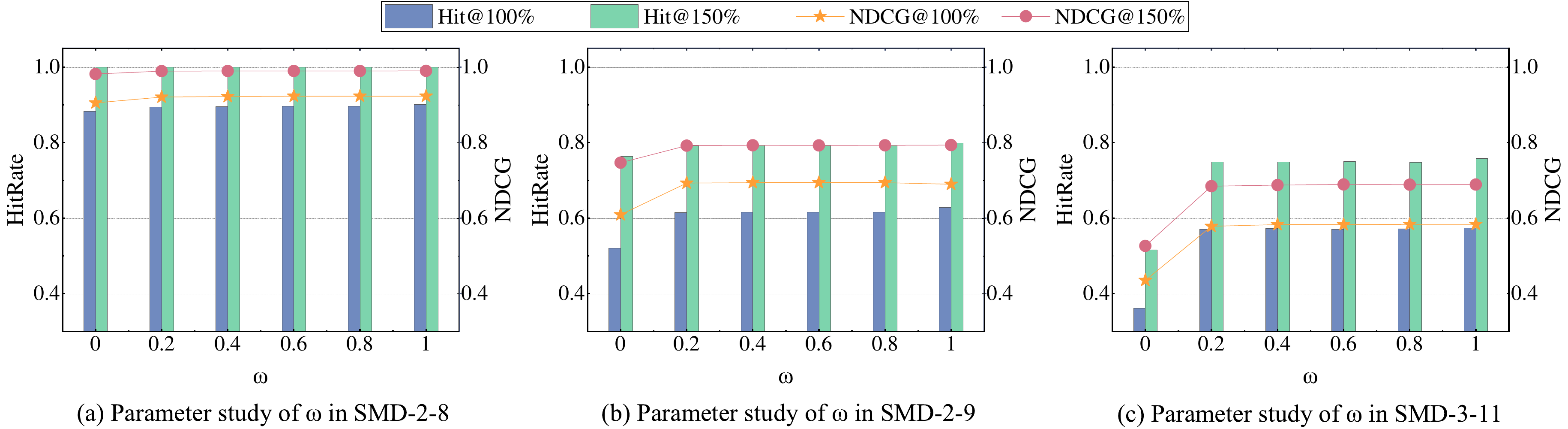}
    \caption{{The weight parameter balancing in the interpretability contribution equation}}
    \label{fig:interpretability-w}
    \vspace{-4mm}
\end{figure*}

{\textbf{Parameter study of $\omega$.}
We perform a parameter study of the fusion weight $\omega$ in Eq.~\ref{Equation 14} over $[0,1]$ on three SMD subsets stratified by interpretability ground-truth length—long (SMD-2-8), medium (SMD-2-9), and short (SMD-3-11). The results in Fig.~\ref{fig:interpretability-w} yield two consistent findings, aligned with Lemmas~\ref{lemma1} and~\ref{lemma2}: (i) performance varies slightly with $\omega$, fluctuations occur mainly for $\omega \in [0, 0.4]$; and (ii) there exists a broad \emph{plateau} (i.e., $\omega \in [0.6, 1.0]$) on which the variable-interpretability ranking remains essentially unchanged. These patterns persist across all three length regimes, indicating robustness both to the choice of $\omega$ and to the ground-truth duration. Therefore, we fix $\omega=0.8$ as a robust default.
}

\begin{table}[t]
\centering
\caption{{Anomaly interpretation performance}} \footnotesize
\setlength{\tabcolsep}{0.7mm}
\begin{tabular}{@{}ccccc@{}}
\toprule
Method      & Hit@100\% & Hit@150\% & NDCG@100\% & NDCG@150\% \\ \midrule
TranAD      & 0.4905   & 0.6634   & 0.5167     & 0.6105     \\
OmniAnomaly & 0.4873        & 0.6119 & 0.5076     & 0.5831     \\
MAD\_GAN    & 0.4607       & 0.6598   & 0.4158     & 0.5345     \\
LSTM\_AD    & 0.4761        & 0.6419   & 0.4474     & 0.5622     \\
Moon        & \textbf{0.5209}        & \textbf{0.7034}        & \textbf{0.5294}     & \textbf{0.6283}     \\ \bottomrule
\end{tabular}
\label{figureinterpretation}
\vspace{-4mm}
\end{table}

\subsection{Performance of Anomaly Interpretation}

We evaluate the anomaly interpretability by using \(\text{Hit@P\%} = \frac{\text{Hit@P\%}}{|GT_{t}|}\), where \( GT_{t} \) is the ground truth array of dimensions that contribute to the anomaly and \( |GT_{t}| \) is its length. Here, Hit@P\%~\cite{OmniAnomaly} represents the number of ground truth dimensions in the top \( P\% \times |GT_{t}| \) of the \( AS \) list, where \( P \) is set to 100 and 150. For example, for an anomaly in a time series with 5 variables, where the ground truth \( GT_t \) is \{2, 3\} and the \( AS \) list generated by \textsc{Moon} is \{2, 1, 3, 5, 4\}, the Hit@100\% is 50\% since only one of the top two variables in \( AS \) matches \( GT_t \). The Hit@150\% is 100\%, as both variables in \( GT_t \) are among the top three in the \( AS \) list. {We also add a new metric the Normalized Discounted Cumulative Gain (NDCG)~\cite{NDCG}, which measures the ranking quality of anomaly interpretation. Similar to HitRate@P\%, NDCG@P\% considers the top $P\% \times |GT_t|$ variables, but it further emphasizes the positions of the ground truth dimensions in the ranking. 
We report the average score over all the samples for each metric, with higher values indicating better interpretability.} We  evaluate \textsc{Moon} exclusively on the SMD dataset, as other datasets lack interpretation labels. 

As shown in Table~\ref{figureinterpretation}, \textsc{Moon} achieves superior anomaly interpretation accuracy, with a Hit@100\% of 0.5294 and a Hit@150\% of 0.7034. In contrast, the state-of-the-art reconstruction-based model \textsc{TranAD} performs significantly worse on the same dataset, with a Hit@100\% of 0.4905 and a Hit@150\% of 0.6634. Reconstruction-based methods such as \textsc{TranAD} rely on thresholding to identify anomalous variables, making their performance highly dependent on detection accuracy and often overlooking the model's influence on individual variables. In contrast, \textsc{Moon} focuses on explaining the model's internal mechanisms rather than merely interpreting its outputs. {Moreover, Moon achieves superior performance on the NDCG@P\% metric, as it emphasizes ranking order rather than simply detecting whether a target is hit. By integrating a scoring mechanism to filter out unreliable explanations, Moon provides more precise interpretations, thereby enhancing both the interpretability and reliability of anomaly detection.}

\begin{figure*}[t]
      \centering
      \includegraphics[width=0.9\linewidth]{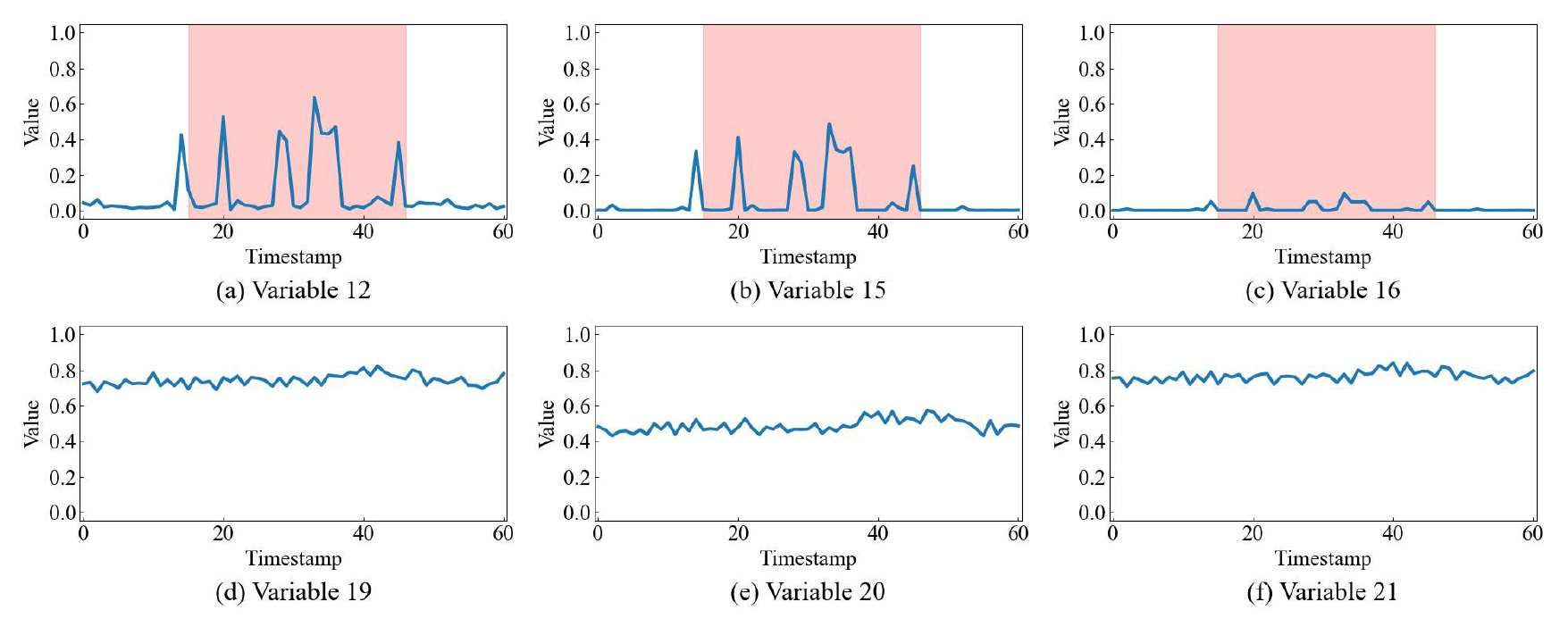}
       \vspace{-4mm}
      \caption{{Time series plots of anomalous variables (12, 15, 16) and partial normal variables (19, 20, 21) in time interval [370, 430] of the SMD 1-4 dataset}}
      \label{figurecase_study_vs}
      \vspace{-3mm}
\end{figure*}

\begin{figure}[t]
      \centering
      \includegraphics[width=\linewidth]{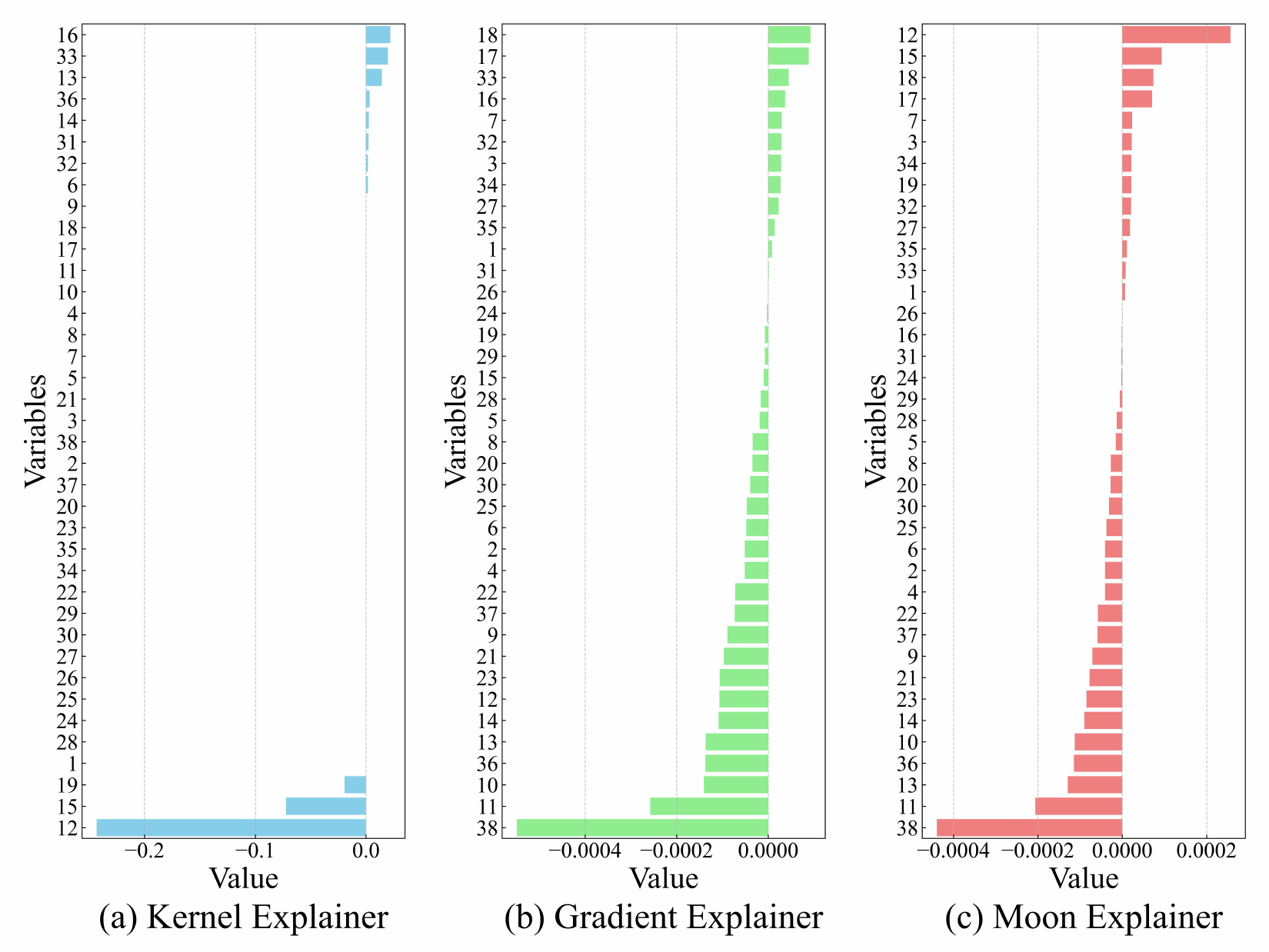}
       \vspace{-3mm}
      \caption{{The relative contribution of variables to the anomalies} }
      \label{figureexplainer}
      \vspace{-3mm}
\end{figure}

\subsection{Case study}
{To validate the performance of our method in real-world scenarios, we conducte a case study using the SMD dataset.
We selected the SMD 1–4 subset and focus our analysis on the time interval $[370, 430]$. Within this interval, the dataset labels the period $[385, 417]$ as anomalous. It is important to note that in SMD, anomaly labels are assigned at the time-point level: if any variable is anomalous at a given timestamp, all variables at this timestamp are marked as anomalous. However, based on the dataset’s provided interpretable labels, we  confirm that the true anomalous variables during this period are variables 12, 15, and 16, corresponding to disk\_svc, disk\_wa, and disk\_wb, which represent the average service time per disk request, the I/O wait time, and the write-back time, respectively. As shown in {Fig.}~\ref{figurecase_study_vs}, the top three subplots display the time-series curves of these anomalous variables, with red shaded regions indicating the actual anomalous intervals. For comparison, we also include several variables (e.g., 19, 20, 21) that remained normal during the same period.}

{Interpretable anomaly detection must not only identify anomalous timestamps but also accurately locate the responsible variables. Once an anomaly is detected, we compute the contribution of each variable by combining the MV-MTF representation with the original TS data. To assess the performance of Moon's explainer, we visualize the contribution results in {Fig.}~\ref{figureexplainer}. As observed, Moon's explainer correctly identifies 2 out of the 3 ground-truth anomalous variables within its top-3 predictions, yielding a Hit@100\% score of 0.667. Finally, engineers can leverage the identified anomalous variables to perform root cause analysis and promptly address system issues. For example, Moon's explainer highlights variables 12 and 15 as anomalous, we can investigate potential disk bottlenecks by examining which processes are heavily utilizing disk I/O and take corresponding optimizations.}


%% file: related_work.tex
\section{Related Work}
\label{Related Work}

\noindent\textbf{Modality Conversion.} Time series data can be transformed into time-frequency representations using various modal conversion methods, such as Short-Time Fourier Transform (STFT)~\cite{STFT}, Wavelet Transform~\cite{WT}, and Hilbert-Huang Transform (HHT)~\cite{HHT}. Each method has distinct characteristics: STFT uses a fixed window scale, which cannot adjust based on frequency; WT is better suited for non-stationary signals; and HHT offers adaptability but with lower resolution.

Moreover, time series data can be encoded into images and feature extraction and classification can be performed using Convolutional Neural Networks (CNN). Common techniques include Gramian Angular Field (GAF)~\cite{GAF-MTF}, Recurrence Plots (RP)~\cite{RP}, Markov Transition Field (MTF)~\cite{GAF-MTF}, Difference Field~\cite{DF}, and Relative Position Matrix~\cite{relative}. GAF converts time series data into cosine and sine Gramian matrices, preserving angular information and capturing dynamic features. RP generates two-dimensional matrices reflecting similarity structures, making it suitable for identifying repeating patterns and periodicity. MTF creates Markov transition matrices, generating images that reflect state transitions and capture dynamic evolution. DF highlights data trend changes by producing images through differential operations. RPM maps time series data to high-dimensional spaces to generate images reflecting relative positional relationships. Among these, MTF is particularly effective for time series data with strong contextual links, as it captures dynamic characteristics and evolutionary trends. However, current research on multivariate MTF techniques remains limited.

To address this gap, this paper proposes MV-MTF  that transforms MTS with contextual links into two-dimensional images. MV-MTF preserves transition relationships between variables over time, effectively enhancing anomaly detection.

\noindent\textbf{Deep Learning Anomaly Detection.}
Deep learning models~\cite{GAF-MTF,USAD, OmiScale-CNN, autoencoder,tranAD} have demonstrated exceptional effectiveness in anomaly detection tasks for complex and large-scale datasets due to their ability to learn intricate data representations and features.

OmniAnomaly~\cite{OmniAnomaly} utilizes a stochastic recurrent neural network combined with a planar normalizing flow to generate reconstruction probabilities. It integrates an improved Peak Over Threshold (POT) method for automated anomaly threshold selection, significantly enhancing performance. However, its training process is lengthy and resource-intensive. MAD-GAN~\cite{MAD-GAN} employs an LSTM-based GAN model to capture time-series distributions through its generator, incorporating both prediction error and discriminator loss into anomaly scoring. This method boosts anomaly detection performance but can also be computationally demanding.

USAD~\cite{USAD} is more efficient than the above methods because it utilizes a simple autoencoder with dual decoders trained in an adversarial framework. This lightweight design significantly reduces training time while maintaining competitive performance.  TranAD~\cite{tranAD} incorporates focus score-based self-conditioning for robust multimodal feature extraction, adversarial training for stability, and Model-Agnostic Meta-Learning (MAML) for effective training on limited data.

CNNs have proven highly effective in extracting and learning features from data, enabling them to capture complex structures. Omni-Scale CNN~\cite{OmiScale-CNN} is particularly adept at handling multi-scale MTS data through its multi-scale feature extraction capabilities. However, it does not support multimodal data, limiting its applicability to more diverse datasets.

To address this limitation, we enhance Omni-Scale CNN by incorporating parameter sharing and a feature fusion module, creating a multimodal anomaly detection model. This enhanced model significantly improves its ability to process and analyze multimodal data, further advancing its effectiveness in anomaly detection.

\noindent\textbf{Explainable Anomaly Detection.}  
Explainable Anomaly Detection (XAD) extracts insights from anomaly detection models, emphasizing data relationships to help users understand the causes of anomalies, thereby enhancing interpretability and usability. In time-series data, XAD methods can be classified into sample-based and model-based methods. 

Sample-based methods explain anomalies by comparing anomalous and normal objects, focusing on local neighborhoods, counterexamples, or contextual anomalies. For example, reconstruction-based methods such as OmniAnomaly~\cite{OmniAnomaly} and TranAD~\cite{tranAD} identify anomalous variables by comparing reconstructed values with original ones.

Model-based methods explain the internal workings or predictions of anomaly detection models. They can be
further divided into white-box and black-box methods. White-box methods are inherently interpretable, while black-box methods use post-hoc techniques, often model-agnostic, to explain predictions~\cite{model-Interpret18, model-Interpret19}. For example, black-box methods such as SHAP values~\cite{model-Interpret17, model-Interpret20} provide insights into feature importance within the model. We integrate sample-based and model-based methods by leveraging SHAP values to enhance interpretability and using normal samples to define the range of anomalous variables.
 

%% file: conclusion.tex
\section{Conclusion}
This paper introduces \textsc{Moon}, an efficient and effective framework for multivariate time series anomaly detection. \textsc{Moon} utilizes the MV-MTF technology to provide more detailed multi-modal information via mode conversion, while introduces Multi-Model-OSCNN  to  effectively
learns and integrates information from different modalities. Additionally, \textsc{Moon} provides user-friendly anomaly interpretability by using SHAP values to rank variables in ascending order of their impact on anomalies and select the top-ranked ones. 
Extensive experiments evaluate the performance of Moon and validate the effectiveness of each component/technique in \textsc{Moon}. compared with existing state-of-the-arts methods, \textsc{Moon} achieves high efficient and accurate anomaly detection and provides  the detail anomaly analysis report for good interpretability. In the future, it is of interest to extend our anomaly detection model to handle other downstream tasks.